\documentclass{article} % For LaTeX2e
\usepackage{conference,times}

\usepackage{amsmath,amsfonts,bm}

\def\eqref#1{equation~\ref{#1}}
\def\1{\bm{1}}

\DeclareMathAlphabet{\mathsfit}{\encodingdefault}{\sfdefault}{m}{sl}
\SetMathAlphabet{\mathsfit}{bold}{\encodingdefault}{\sfdefault}{bx}{n}

\DeclareMathOperator*{\argmax}{arg\,max}

\usepackage{url}
\usepackage{booktabs}
\usepackage{multirow}
\usepackage{amsmath}
\usepackage{amssymb}
\usepackage{amsthm}
\usepackage{mathtools}
\usepackage{enumitem}
\usepackage{xcolor}
\usepackage{pifont}
\usepackage{tikz}
\usepackage{wrapfig}
\usepackage{graphicx}
\usetikzlibrary{calc}
\usetikzlibrary{decorations.pathreplacing}
\usetikzlibrary{arrows.meta}

\usepackage{hyperref}
\hypersetup{colorlinks=true,linkcolor=black,citecolor=blue,urlcolor=blue}
\theoremstyle{plain}
\newtheorem{theorem}{Theorem}
\newtheorem{lemma}[theorem]{Lemma}
\newtheorem{proposition}[theorem]{Proposition}
\newtheorem{corollary}[theorem]{Corollary}
\theoremstyle{definition}

\theoremstyle{remark}
\newtheorem{remark}[theorem]{Remark}

\newcommand{\Prob}{\mathbb{P}}
\newcommand{\indic}{\mathbf{1}}

\newcommand{\Yhead}{\widehat{Y}^{\mathrm{head}}}
\newcommand{\Yproto}{\widehat{Y}^{\mathrm{prot}}}
\newcommand{\Ydef}{\widehat{Y}^{\mathrm{def}}}

\newcommand{\ASR}{\mathrm{ASR}}

\newcommand{\Yrep}{\widehat{Y}^{\mathrm{rep}}}
\title{RePAIR: Inverting Backdoor Attacks in Semantic Segmentation}

\author{%
Arnold Brosch
\\
Heinrich Heine University Düsseldorf \\
\texttt{arnold.brosch@hhu.de} \\
\And
Kira Maag \\
Heinrich Heine University Düsseldorf\\
\texttt{kira.maag@hhu.de}\\
}

\iclrfinalcopy % Uncomment for camera-ready version, but NOT for submission.
\begin{document}

\maketitle

\begin{abstract} 
Despite significant advances in computer vision, semantic segmentation models remain vulnerable to backdoor attacks, posing a security risk. These attacks redirect pixels of a chosen victim class to an attacker-specified target class when a trigger is present, while preserving clean predictions.
Existing defenses are largely derived from image classification, while dedicated segmentation defenses remain scarce and typically require retraining or model modification. 
We introduce RePAIR, a post-hoc, label-free, attack-unaware defense based on the observation that backdoor-induced shifts can be stronger in the classifier head than in internal representations, leaving recoverable victim-class semantics. 
RePAIR infers the victim-target relation from an unlabeled, potentially contaminated image pool, constructs a safety-approved semantic reference, and localizes suspicious pixels by comparing head predictions with internal semantic evidence.
It then applies feature-space and output corrections toward the inferred victim semantics, while conservative localization and output constraints limit changes to benign predictions. 
We model the attack as adversarial probability transport and repair as its approximate reversal, deriving bounds and sufficient conditions for localization, residual attack success, clean-side changes, defense-induced target predictions, and attack-unaware pair identification. 
Experiments across segmentation architectures and attack settings show that RePAIR substantially reduces attack success while largely preserving clean performance.
\end{abstract}

\section{Introduction}
\label{sec:intro}

Semantic segmentation assigns a semantic class to every pixel proving perception in safety-critical systems such as automated driving \citep{cakir2022semanticsegmentationautonomousdriving} and medical imaging \citep{taghanaki2024deepsemanticsegmentationnatural}. 
Semantic segmentation networks are vulnerable to backdoor attacks that preserve normal performance on benign inputs while inducing targeted failures when a trigger is present \citep{li2021hidden}. 
For instance, a backdoor can cause a vehicle to disappear from the semantic map by relabeling its pixels as road, without affecting the prediction for the rest of the image.
Semantic segmentation backdoors have developed from fixed trigger patterns and predefined victim–target mappings \citep{li2021hidden} toward increasingly flexible attack settings. %\cite{zhang2026poisoningpixelsrevisitingbackdoor}.
Existing attacks allow the victim class to depend on trigger placement \citep{mao2023object}, exploit contextual relationships between semantic classes \citep{abbasi2025consegcontextualbackdoorattack}, or automatically optimize triggers and attack configurations \citep{lan2024influencer}. 
These attacks demonstrate that backdoors can induce targeted and spatially localized prediction changes under 
flexible conditions, making the development of robust defenses challenging.

While attacks on dense prediction tasks have received increasing attention, effective defenses remain comparatively underexplored. 
Detection-oriented methods can identify suspicious inputs or models but do not necessarily remove the underlying backdoor \citep{Liu_2023_CVPR}. 
Mitigation approaches based on retraining (e.g. fine-tuning and pruning) often require additional clean data and may alter the model's behavior \citep{liu2018fine}. 
For semantic segmentation, effective defenses should selectively correct attacked pixels while limiting changes to unrelated regions, motivating selective repair rather than global model modification.

\begin{figure}[t]
\centering
\includegraphics[width=0.245\linewidth]{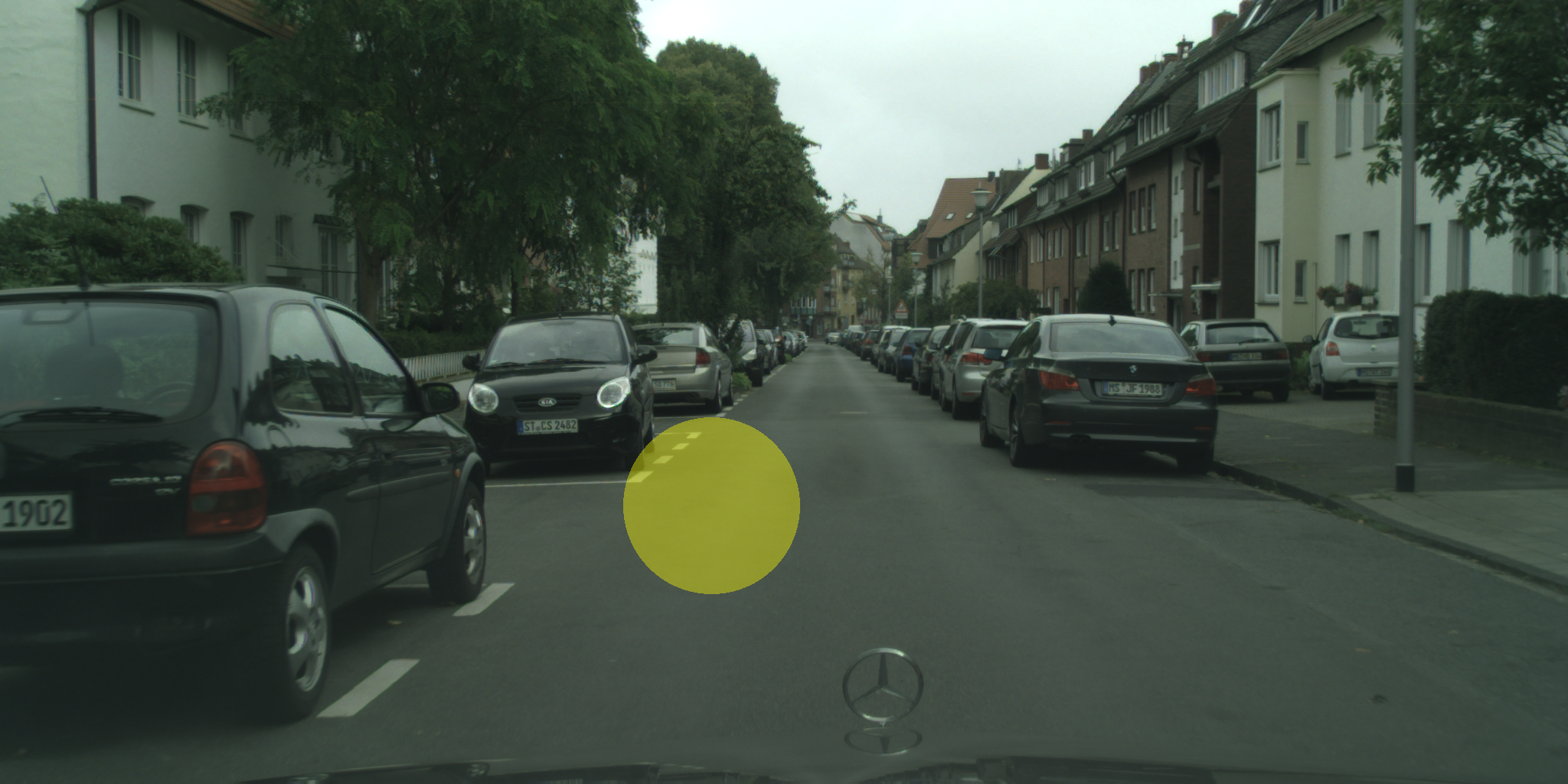}
\includegraphics[width=0.245\linewidth]{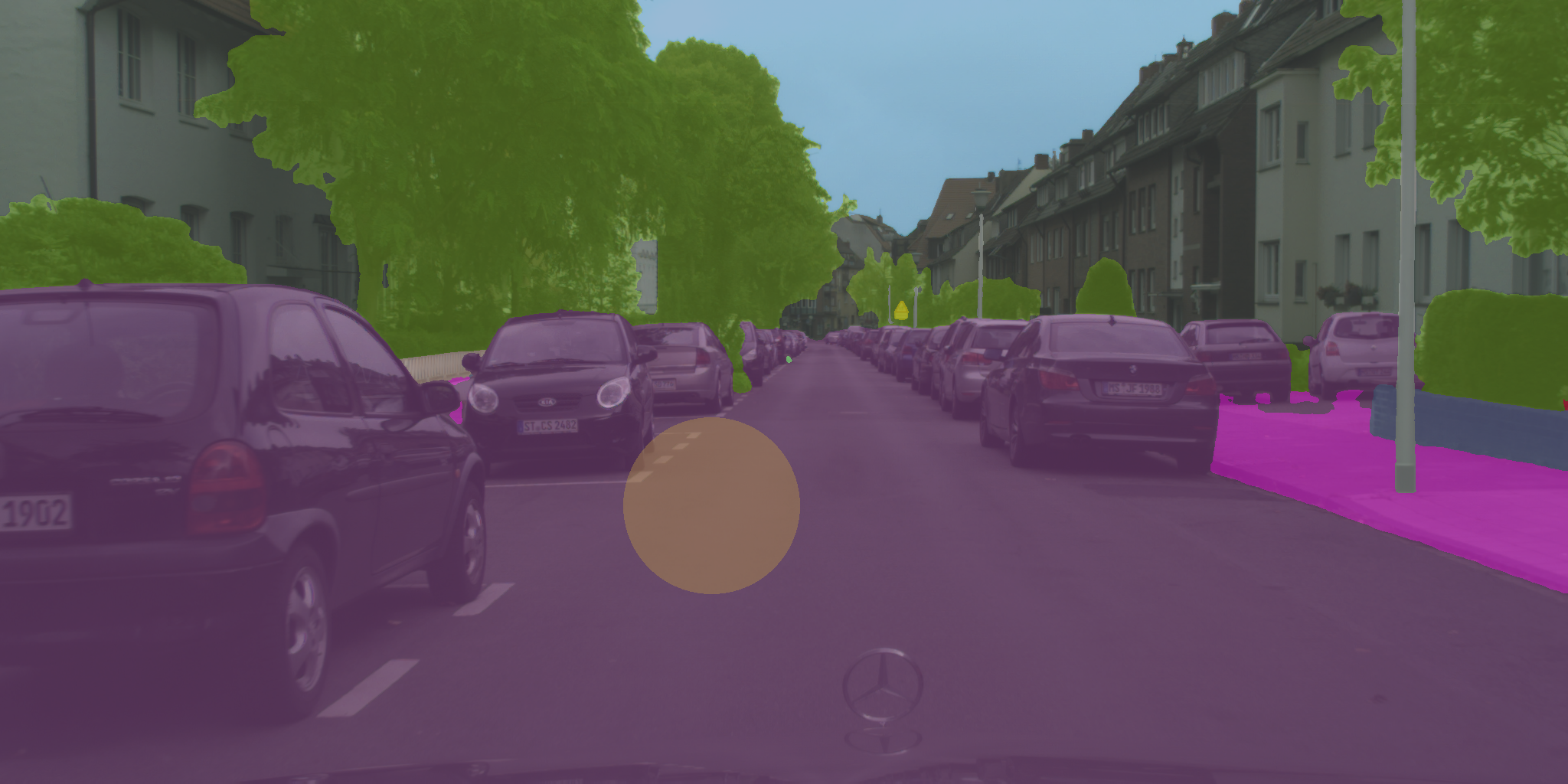}
\includegraphics[width=0.245\linewidth]{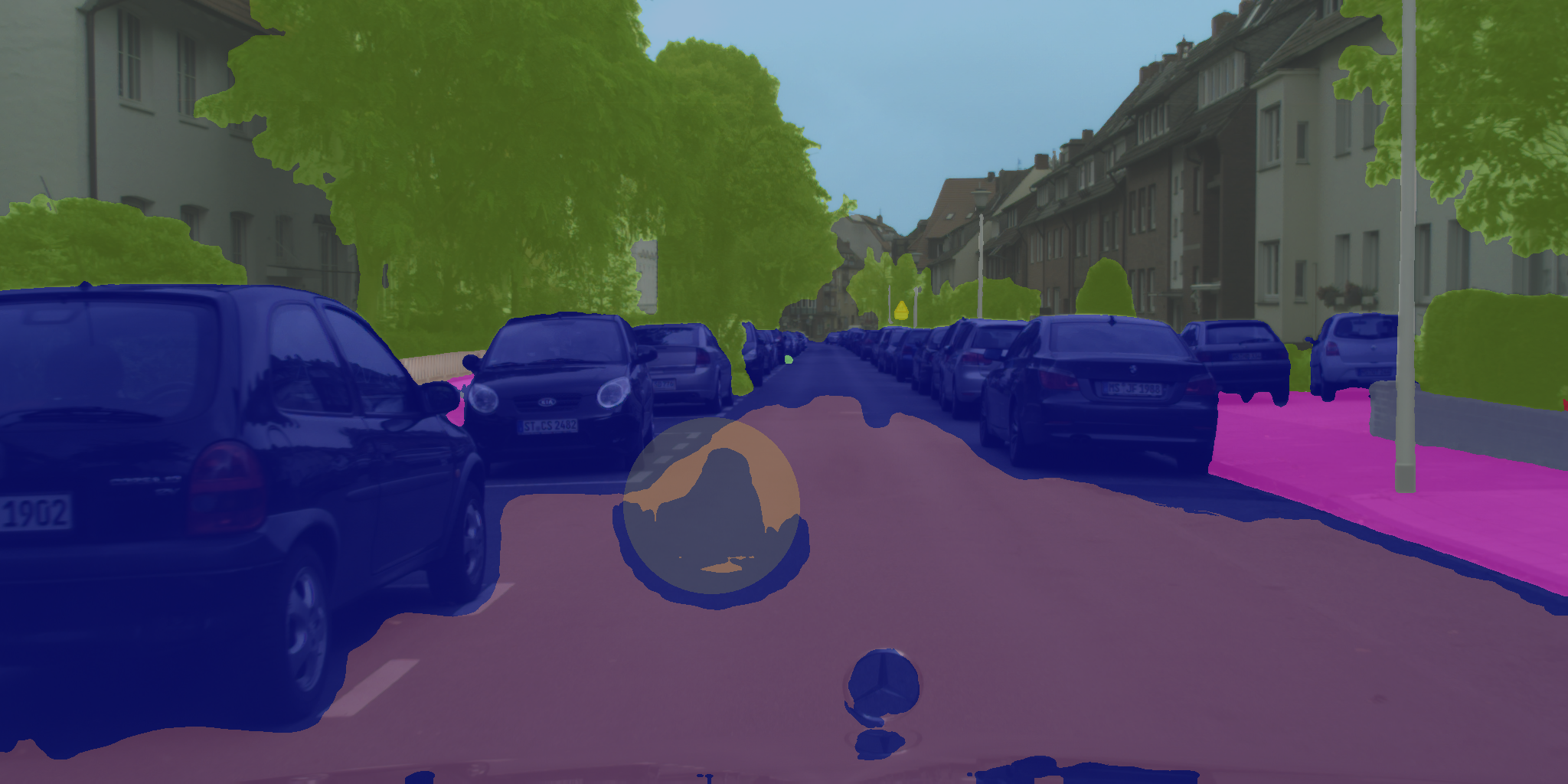}
\includegraphics[width=0.245\linewidth]{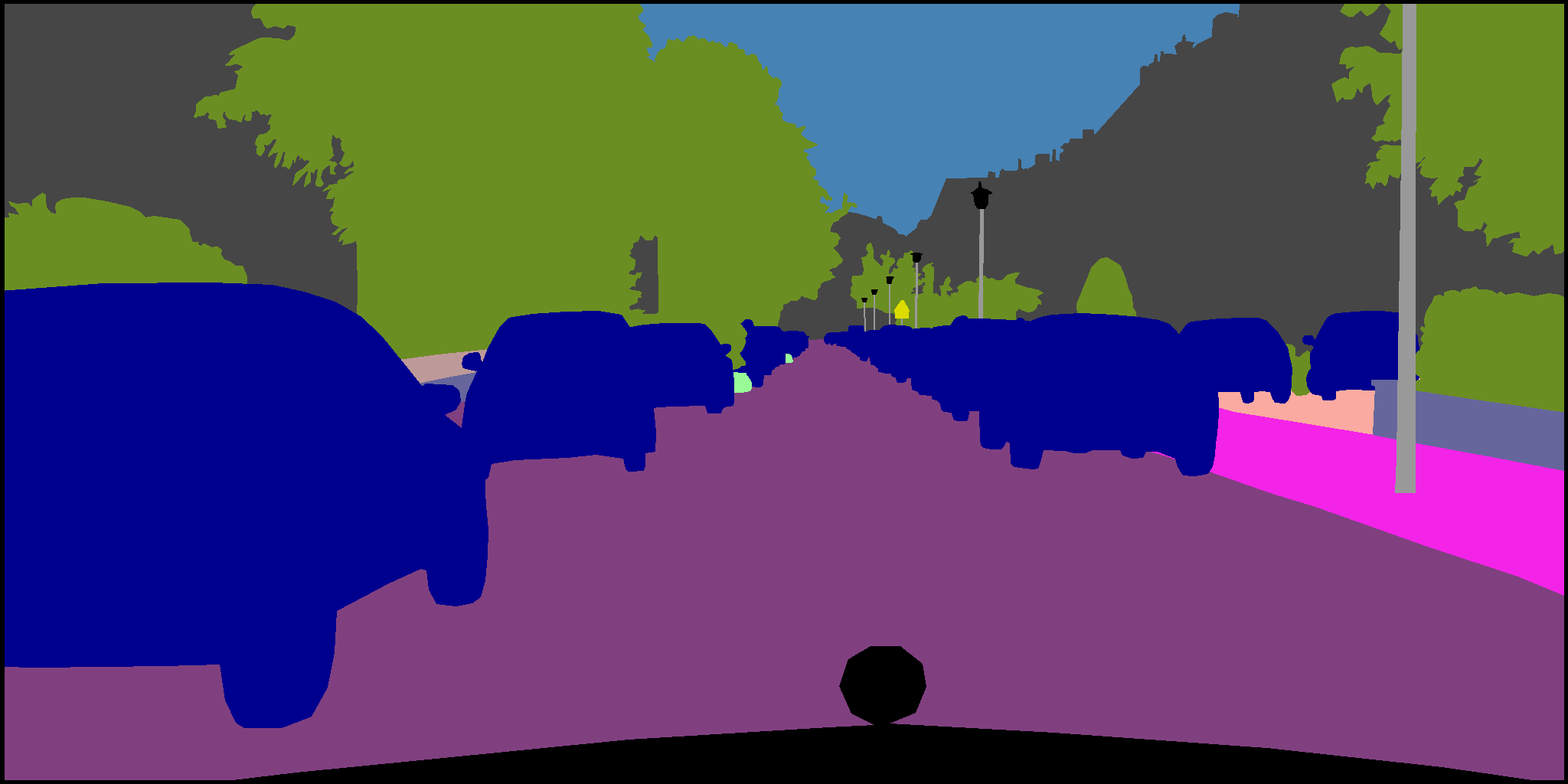}
\caption{From left to right: poisoned image, prediction under the backdoor attack, prediction after applying our defense and ground truth segmentation. }
\label{fig:example_intro}
\end{figure}

In this paper, we introduce a \textbf{Re}presentation-guided pixel-wise attack Re\textbf{pair} (RePAIR) against backdoor attacks on semantic segmentation models. 
We focus on a key empirical observation: the backdoor-induced shift can be substantially stronger in the final task decision than in an appropriately chosen internal semantic representation, even when the encoder itself is affected by the backdoor. 
Victim semantics can therefore remain partially recoverable even when
the head predicts the attacker-selected target. The representation that exposes this
state is architecture-dependent. 
E.g. for DeepLabV3+ \citep{Chen2018_deeplab} with a ResNet-50 backbone, the final backbone stage immediately before atrous spatial pyramid pooling provides a useful reference. 
For SegFormer \citep{xie2021segformersimpleefficientdesign} with a MiT-B5 backbone, no single block is sufficient. 
Victim evidence is distributed across several internal blocks and is reconstructed by a learned
multi-layer logistic read-out. 

This observation has a natural geometric interpretation.
For each pixel, the predicted class probabilities form a point in the probability simplex, where each coordinate represents the probability assigned to one semantic class. 
An idealized targeted backdoor shifts probability mass from a victim class to an attacker-selected target class, moving the prediction toward the target and away from the victim while remaining within the same simplex. 
Natural segmentation errors can also shift probability mass between classes (typically not concentrated in one direction), but a backdoor induces a systematic flow from one semantic class to another.

This directed view can be described by a local transport operator acting on the class-probability distribution. 
RePAIR approximates the corresponding reverse transport from surviving victim evidence and model-derived reference geometry, while treating feasible victim-target pairs as hypotheses and identifying the attacked relation without prior knowledge of either class.

A correct victim–target pair alone is not sufficient for a safe repair.  
The reference data used for repair must also be protected from poisoned samples, otherwise, the defense may learn the very distortion it is meant to remove.
We therefore first identify candidate clean images and then check whether they provide reliable evidence of the victim class before using them for repair. 
An image is rejected if its internal victim evidence is severely degraded, or if a moderate degradation is accompanied by an unusually strong shift from the victim class to the target class. 
Rejected images are not replaced merely to meet a desired reference-set size. 
If too little reliable victim evidence remains, the defense does not construct a repair state.   
Together, these components form our method for selective and reliable backdoor repair.
The source code of our method is publicly available at \url{https://github.com/A-Brosch/RePAIR}.  
Our contributions can be summarized as follows:
\begin{itemize}
\item We introduce a post-hoc repair method for backdoored semantic segmentation models that reconstructs victim-consistent semantics and selectively corrects attacked pixels, without retraining or globally modifying the deployed network.
\item We develop a theoretical framework that models targeted segmentation backdoors as directed victim-to-target transport in the class-probability simplex. This view links surviving internal evidence to localization, residual attack success, attack-unaware pair identification, and clean-side safety. 
\item Our method for generating a reliable repair reference is attack-unaware and operates on potentially compromised data. It identifies consistent victim-to-target evidence and screens candidate clean-like images with a one-way victim-safety firewall.
\item We develop architecture-specific semantic references for both convolutional- and transformer-based segmentation models, showing that the information required for repair may be concentrated in a single representation or distributed across multiple internal layers.
\end{itemize}

% ======================================================
\section{Related Work}
\label{sec:related}

\paragraph{Backdoors Attacks.}
The first work adapted the BadNets attack \citep{Gu2019BadNets} from image classification to semantic segmentation, using either a black stripe at the image border as a non-semantic trigger or a semantic class present in the image as a trigger \citep{li2021hidden}. Once activated, the attack relabeled all pixels of a victim class as a target class.
The Influencer Backdoor Attack (IBA, \citep{lan2024influencer}) places a trigger on non-victim pixels while relabeling the victim class as a predefined target. 
It further introduces Nearest Neighbor Injection (NNI) and Pixel Random Labeling (PRL), which respectively place the trigger near the victim region and add random label changes to non-victim pixels to strengthen the attack's reliance on global context. 
The Object-Free Backdoor Attack (OFBA, \citep{mao2023object}) places the trigger on the target object and varies victim classes and trigger positions during training, allowing the victim class to be selected by trigger placement at inference time while keeping the target class fixed. 
The Contextual Segmentation attack (ConSeg) uses a natural semantic trigger, such as a motorcycle, instead of an artificial pattern \citep{abbasi2025consegcontextualbackdoorattack}. It relabels the victim class as the target and adds a small number of co-occurring target-related classes within the victim region to strengthen the contextual association.
BADSEG \citep{zhang2026poisoningpixelsrevisitingbackdoor} broadens the attack space to four attack types: Object-to-Object (O2O), Object-to-Background (O2B), Background-to-Object (B2O), and Background-to-Background (B2B), along with instance-level and conditional variants. 
A surrogate model optimizes trigger parameters and selects victim-target pairs using feature-space class centroids.

\textbf{Defense Methods.} \hspace{0.5ex}
Work on defenses against backdoor attacks in semantic segmentation remains limited, and existing studies therefore mainly adapt defenses originally developed for image classification. 
Test-time detection methods such as STRIP \citep{gao2019strip} and TeCo \citep{Liu_2023_CVPR} detect suspicious inputs by measuring prediction entropy under image superposition or prediction consistency under image corruptions, respectively.  
Similarly, Zeng et al.~\citep{Zeng_2021_ICCV} detect anomalous trigger-related patterns in frequency domain \citep{ahmed1974discrete}. 
BEATRIX \citep{beatrix} detects anomalous internal representations using class-conditional higher-order Gram matrices \citep{zhang2026poisoningpixelsrevisitingbackdoor}. 
These methods aim to identify suspicious inputs but do not repair the affected predictions or remove the underlying backdoor. 
Pruning \citep{liu2018fine} removes neurons with low activation on benign inputs and fine-tunes the remaining network on clean data to eliminate dormant backdoor components and recover clean performance.
Fine-tuning \citep{Sha:He:Berrang:Humbert:Zhang:2022} instead directly retrains the potentially compromised model on clean data to weaken the learned association between trigger and target.
Both methods were also evaluated as defenses against segmentation backdoors, but both require access to clean data for mitigation \citep{lan2024influencer,zhang2026poisoningpixelsrevisitingbackdoor}.
Segment Cleanse \citep{mao2023object} adapts Neural Cleanse \citep{wang2019neural} to semantic segmentation by reconstructing object-level trigger candidates and
mitigating detected backdoors through retraining, likewise requiring model modification and clean data. 
Anti Backdoor Learning (ABL, \citep{Li2021AntiBackdoor}) exploits differences in training loss dynamics to identify suspicious samples and reduce their influence through gradient ascent. 
While these approaches identify or suppress backdoor behavior, they do not enable post-hoc, selective repair of affected pixels in a deployed model.

RePAIR differs from detection-based methods (e.g.\ STRIP, TeCo, frequency-domain defenses, and BEATRIX), which identify suspicious inputs, samples, or internal representations but do not repair the affected predictions. 
In contrast to model-modification approaches such as pruning, fine-tuning, and Segment Cleanse, our method does not retrain or otherwise alter the compromised model.  
It is also applied post-hoc, unlike ABL, which operates during training. 
Finally, RePAIR requires neither ground truth labels nor an external clean reference set, instead deriving the information needed for selective repair from the model and available data while leaving unaffected predictions unchanged.

\section{Theoretical Framework}
\label{sec:theory}
We view a targeted backdoor in semantic segmentation as a trigger-induced, spatially selective change in dense predictions. 
Semantic segmentation predicts a probability distribution over semantic classes for each pixel, allowing a backdoor to selectively alter the predictions of pixels belonging to a particular victim class. 
During training, the attacker introduces triggered images with modified pixel-wise labels, assigning victim-class pixels to an attacker-selected target class. 
When activated, the backdoor causes pixels belonging to a victim class to be relabeled as an attacker-selected target class, while leaving other pixels largely unaffected. 
We treat the images as samples of a stochastic process, we further assume the neural network to learn the induced distribution during training.
In order to change the argmax of the attacked pixels' prediction the trigger has to lead to redistribution of probabilisitc mass from one atom to another (assumed to be mostly from victim to target class). 
The possibility of recognizing and ``inverting'' this transport yields two effects, the identification of victim and target class and nullyfication of the triggers effect, thus rendering it use- and harmless.

We model a targeted segmentation backdoor as a spatially selective victim-to-target distortion, corresponding to a directed redistribution of class probability from a victim class $v$ to a target class $t\neq v$. 
A successful backdoor may drive the final segmentation head toward the target class while an architecture-specific internal reference still retains evidence of the original victim semantics. 
This disagreement between the final prediction and the internal semantic evidence provides the signal used for localization and repair. 
RePAIR exploits this mismatch without assuming encoder invariance, semantic ground truth, trigger presence, or prior knowledge of the attacked pair.

\begin{figure*}
    \centering
    \resizebox{0.7\textwidth}{!}{% Required:
% \usepackage{amsmath,amssymb}
% \usepackage{tikz}
% \usetikzlibrary{arrows.meta}

\begin{tikzpicture}[
    every node/.style={font=\small},
    axis/.style={->, thick},
    stemvictim/.style={very thick, blue!75!black},
    stemtarget/.style={very thick, red!75!black},
    atomvictim/.style={
        circle,
        fill=blue!75!black,
        inner sep=1.7pt
    },
    atomtarget/.style={
        circle,
        fill=red!75!black,
        inner sep=1.7pt
    },
    supportnode/.style={
        circle,
        draw=black!60,
        fill=white,
        minimum size=20pt,
        inner sep=0pt
    },
    retained/.style={
        ->,
        gray!60,
        >=Latex,
        line width=0.9pt
    },
    attackflow/.style={
        ->,
        teal!70!black,
        >=Latex,
        line width=2.3pt
    },
    repairflow/.style={
        ->,
        orange!85!black,
        >=Latex,
        line width=2.3pt
    }
]

% ============================================================
% PANEL (a)
% ============================================================

\node[font=\small\bfseries]
    at (8.2,5.00)
    {(a) Backdoor-induced redistribution};

\node
    at (8.2,4.52)
    {$
      q=p^\star+\xi d_{vt},
      \qquad
      d_{vt}=e_t-e_v
     $};

% ------------------------------------------------------------
% (a1) Clean output
% ------------------------------------------------------------

\begin{scope}[shift={(0,0)}]

    \node[
        blue!75!black,
        font=\small\bfseries
    ] at (2.4,3.65)
        {Clean output $p^\star$};

    \draw[axis]
        (0,0) -- (4.9,0)
        node[right] {class};

    \draw[axis]
        (0,0) -- (0,2.75)
        node[above] {probability};

    \foreach \x/\h/\lab/\mass in {
        0.5/0.50/{c_1}/{0.10},
        1.4/1.75/{c_v}/{0.35},
        2.3/0.75/{c_3}/{0.15},
        3.2/1.25/{c_t}/{0.25},
        4.1/0.75/{c_5}/{0.15}}
    {
        \draw[stemvictim]
            (\x,0) -- (\x,\h);

        \node[atomvictim]
            at (\x,\h) {};

        \node[
            above=3pt,
            blue!75!black
        ] at (\x,\h)
            {$\mass$};

        \node[
            below=3pt
        ] at (\x,0)
            {$\lab$};
    }

    \node[left=3pt] at (0,0)   {$0$};
    \node[left=3pt] at (0,1.0) {$0.2$};
    \node[left=3pt] at (0,2.0) {$0.4$};

    \node[
        blue!75!black
    ] at (2.4,-0.72)
        {$
        \operatorname*{arg\,max}_{k}
        p^\star_k=c_v
        $};

\end{scope}

% ------------------------------------------------------------
% (a2) Transport coupling
% ------------------------------------------------------------

\begin{scope}

    \node[
        font=\small\bfseries
    ] at (8.2,3.65)
        {transport coupling $\gamma_\omega$};

    \node[
        blue!75!black
    ] at (6.65,3.15)
        {clean};

    \node[
        red!75!black
    ] at (9.75,3.15)
        {triggered};

    \foreach \i/\lab/\y in {
        1/{c_1}/2.60,
        2/{c_v}/2.05,
        3/{c_3}/1.50,
        4/{c_t}/0.95,
        5/{c_5}/0.40}
    {
        \node[supportnode]
            (aL\i)
            at (6.65,\y)
            {$\lab$};

        \node[supportnode]
            (aR\i)
            at (9.75,\y)
            {$\lab$};
    }

    \foreach \i in {1,...,5}{
        \draw[retained]
            (aL\i) -- (aR\i);
    }

    \draw[attackflow]
        (aL2)
        to[bend right=18]
        node[
            pos=0.48,
            below=4pt,
            fill=white
        ]
        {$\xi=0.17$}
        (aR4);

    % Compact legend
    \draw[retained]
        (6.55,-0.38) -- (7.20,-0.38);

    \node[
        anchor=west
    ] at (7.30,-0.38)
        {retained mass};

    \draw[attackflow]
        (6.55,-0.82) -- (7.20,-0.82);

    \node[
        anchor=west
    ] at (7.30,-0.82)
        {$v\!\rightarrow\!t$ transfer};

\end{scope}

% ------------------------------------------------------------
% (a3) Triggered output
% ------------------------------------------------------------

\begin{scope}[shift={(11.5,0)}]

    \node[
        red!75!black,
        font=\small\bfseries
    ] at (2.4,3.65)
        {Triggered output $q$};

    \draw[axis]
        (0,0) -- (4.9,0)
        node[right] {class};

    \draw[axis]
        (0,0) -- (0,2.75)
        node[above] {probability};

    \foreach \x/\h/\lab/\mass in {
        0.5/0.50/{c_1}/{0.10},
        1.4/0.90/{c_v}/{0.18},
        2.3/0.75/{c_3}/{0.15},
        3.2/2.10/{c_t}/{0.42},
        4.1/0.75/{c_5}/{0.15}}
    {
        \draw[stemtarget]
            (\x,0) -- (\x,\h);

        \node[atomtarget]
            at (\x,\h) {};

        \node[
            above=3pt,
            red!75!black
        ] at (\x,\h)
            {$\mass$};

        \node[
            below=3pt
        ] at (\x,0)
            {$\lab$};
    }

    \node[left=3pt] at (0,0)   {$0$};
    \node[left=3pt] at (0,1.0) {$0.2$};
    \node[left=3pt] at (0,2.0) {$0.4$};

    \node[
        red!75!black
    ] at (2.4,-0.72)
        {$
        \operatorname*{arg\,max}_{k}
        q_k=c_t
        $};

\end{scope}

% ============================================================
% SEPARATOR
% ============================================================

\draw[gray!50]
    (-0.2,-1.55) -- (16.6,-1.55);

% ============================================================
% PANEL (b)
% ============================================================

\node[
    font=\small\bfseries
] at (8.2,-2.15)
    {(b) Conceptual correction direction};

\node
    at (8.2,-2.65)
    {$
      q\rightsquigarrow\widetilde p,
      \qquad
      -d_{vt}=e_v-e_t
     $};

% ------------------------------------------------------------
% (b1) Triggered output
% ------------------------------------------------------------

\begin{scope}[shift={(0,-7.05)}]

    \node[
        red!75!black,
        font=\small\bfseries
    ] at (2.4,3.65)
        {Triggered output $q$};

    \draw[axis]
        (0,0) -- (4.9,0)
        node[right] {class};

    \draw[axis]
        (0,0) -- (0,2.75)
        node[above] {probability};

    \foreach \x/\h/\lab/\mass in {
        0.5/0.50/{c_1}/{0.10},
        1.4/0.90/{c_v}/{0.18},
        2.3/0.75/{c_3}/{0.15},
        3.2/2.10/{c_t}/{0.42},
        4.1/0.75/{c_5}/{0.15}}
    {
        \draw[stemtarget]
            (\x,0) -- (\x,\h);

        \node[atomtarget]
            at (\x,\h) {};

        \node[
            above=3pt,
            red!75!black
        ] at (\x,\h)
            {$\mass$};

        \node[
            below=3pt
        ] at (\x,0)
            {$\lab$};
    }

    \node[left=3pt] at (0,0)   {$0$};
    \node[left=3pt] at (0,1.0) {$0.2$};
    \node[left=3pt] at (0,2.0) {$0.4$};

    \node[
        red!75!black
    ] at (2.4,-0.72)
        {$
        \operatorname*{arg\,max}_{k}
        q_k=c_t
        $};

\end{scope}

% ------------------------------------------------------------
% (b2) Conceptual reverse coupling
% ------------------------------------------------------------

\begin{scope}[shift={(0,-7.05)}]

    \node[
        font=\small\bfseries
    ] at (8.2,3.65)
        {conceptual reverse coupling $\gamma_\omega^\top$};

    \node[
        red!75!black
    ] at (6.65,3.15)
        {triggered};

    \node[
        blue!75!black
    ] at (9.75,3.15)
        {repaired};

    \foreach \i/\lab/\y in {
        1/{c_1}/2.60,
        2/{c_v}/2.05,
        3/{c_3}/1.50,
        4/{c_t}/0.95,
        5/{c_5}/0.40}
    {
        \node[supportnode]
            (bL\i)
            at (6.65,\y)
            {$\lab$};

        \node[supportnode]
            (bR\i)
            at (9.75,\y)
            {$\lab$};
    }

    \foreach \i in {1,...,5}{
        \draw[retained]
            (bL\i) -- (bR\i);
    }

    \draw[repairflow]
        (bL4)
        to[bend left=18]
        node[
            pos=0.48,
            above=4pt,
            fill=white
        ]
        {$\xi=0.17$}
        (bR2);

    \draw[retained]
        (6.55,-0.38) -- (7.20,-0.38);

    \node[
        anchor=west
    ] at (7.30,-0.38)
        {retained mass};

    \draw[repairflow]
        (6.55,-0.82) -- (7.20,-0.82);

    \node[
        anchor=west
    ] at (7.30,-0.82)
        {$t\!\rightarrow\!v$ correction};

\end{scope}

% ------------------------------------------------------------
% (b3) Idealized repaired output
% ------------------------------------------------------------

\begin{scope}[shift={(11.5,-7.05)}]

    \node[
        blue!75!black,
        font=\small\bfseries
    ] at (2.4,3.65)
        {Idealized repaired output $\widetilde p$};

    \draw[axis]
        (0,0) -- (4.9,0)
        node[right] {class};

    \draw[axis]
        (0,0) -- (0,2.75)
        node[above] {probability};

    \foreach \x/\h/\lab/\mass in {
        0.5/0.50/{c_1}/{0.10},
        1.4/1.75/{c_v}/{0.35},
        2.3/0.75/{c_3}/{0.15},
        3.2/1.25/{c_t}/{0.25},
        4.1/0.75/{c_5}/{0.15}}
    {
        \draw[stemvictim]
            (\x,0) -- (\x,\h);

        \node[atomvictim]
            at (\x,\h) {};

        \node[
            above=3pt,
            blue!75!black
        ] at (\x,\h)
            {$\mass$};

        \node[
            below=3pt
        ] at (\x,0)
            {$\lab$};
    }

    \node[left=3pt] at (0,0)   {$0$};
    \node[left=3pt] at (0,1.0) {$0.2$};
    \node[left=3pt] at (0,2.0) {$0.4$};

    \node[
        blue!75!black
    ] at (2.4,-0.72)
        {$
        \operatorname*{arg\,max}_{k}
        \widetilde p_k=c_v
        $};

\end{scope}

\end{tikzpicture}}%
    \caption{Targeted backdoor as directed victim-to-target transport. The
    prediction may move toward the target while internal representations retain
    victim evidence, which RePAIR uses for localization and repair. The forward
    operator and its coupling describe the same directed mass transfer at the
    marginal and joint-measure levels, respectively.}
    \label{fig:backdoor_transport}
\end{figure*}
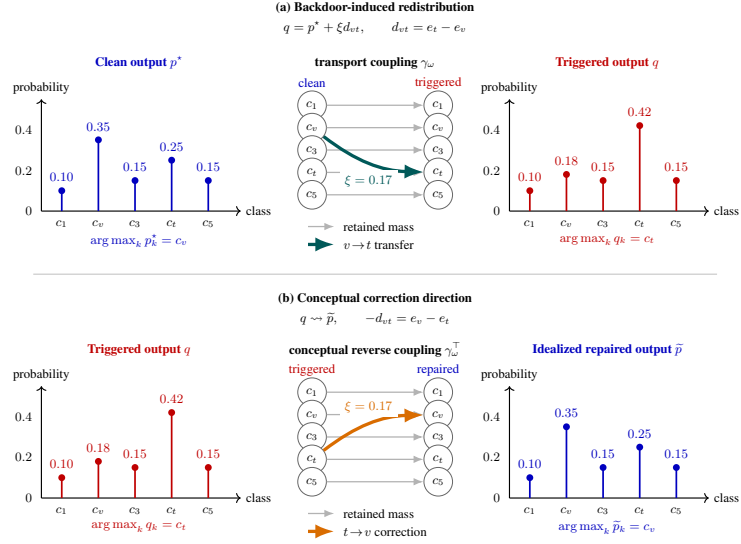

\textbf{Directed distortion.} \hspace{0.5ex}
Let $p^\star\in\Delta^{K-1}$ denote the clean reference probability vector over $K$ semantic classes and $q$ its triggered counterpart predicted by a semantic segmentation network.
The attack is modeled locally as a transfer of probability mass $\xi\geq0$ from the victim class $v$ toward the target class $t$ 
\begin{equation}
    q=T_{\xi}(p^\star)
    =p^\star+\xi(e_t-e_v) \enspace,
    \qquad
    0\leq\xi\leq p_v^\star \enspace,
    \label{eq:theory_transport_operator}
\end{equation}  
where $e_c$ is the unit vector associated with class $c$. 
The geometric interpretation shown is in Fig.~\ref{fig:backdoor_transport}. 
Thus the victim probability decreases
by $\xi$, the target probability increases by $\xi$, and the victim-over-target margin $\psi_{vt}(p)=p_v-p_t$ satisfies
\begin{equation}
    \psi_{vt}(q)=\psi_{vt}(p^\star)-2\xi \enspace.
    \label{eq:theory_margin_shift}
\end{equation}
Assuming $p_v^\star>p_t^\star$, pairwise preference changes once
$\xi>(p_v^\star-p_t^\star)/2$. A complete class change additionally requires
the target coordinate to exceed all remaining classes.
This is an analytical description of the dominant directed failure, not a claim that the poisoned network is globally one-dimensional or that RePAIR observes the unknown clean $p^\star$ or the transported mass $\xi$.

\textbf{Inverse operator and transport coupling.} \hspace{0.5ex}
The local model in Equation~\eqref{eq:theory_transport_operator} makes the notion of \emph{inverting} precise. 
If $\xi$ were known and $q$ lay in the
feasible image of $T_\xi$, the inverse marginal map would be
\begin{equation}
    T_\xi^{-1}(q)
    =q+\xi(e_v-e_t)\enspace.
    \label{eq:theory_inverse_operator}
\end{equation}
The same forward displacement can also be represented by a one-edge coupling $\gamma_\xi$. 
With $(\gamma_\xi)_{vv}=p_v^\star-\xi$, $(\gamma_\xi)_{vt}=\xi$, and $(\gamma_\xi)_{kk}=p_k^\star$ for $k\neq v$, its marginals satisfy
\begin{equation}
    \gamma_\xi\mathbf 1=p^\star \enspace,
    \qquad
    \gamma_\xi^\top\mathbf 1=q\enspace.
    \label{eq:theory_coupling_marginals}
\end{equation}
These are related but distinct descriptions of reversal.
The inverse operator $T_\xi^{-1}$ maps the displaced marginal $q$ back toward $p^\star$, whereas the transposed coupling $\gamma_\xi^\top$ exchanges the forward coupling's
marginals and therefore represents the corresponding reverse mass flow as a joint measure. 
The transposed coupling is not itself the inverse operator, and
no optimal-transport claim is implied without specifying a transport cost.
RePAIR observes neither $\xi$ nor the hidden attack operator or coupling. 
Its repair instead approximates the reverse victim-directed transport from the discovered pair, surviving internal victim evidence, and safety-approved reference geometry. 
This is the sense in which RePAIR \emph{inverts} the
backdoor effect.

\textbf{Surviving internal evidence and localization.} \hspace{0.5ex}
Let $A=\{\Yhead=t\}$ denote successful baseline attack on a triggered pixel of the victim class, where $\Yhead$ is the prediction of the undefended segmentation head.  
Let $\rho_v$ indicate whether the internal representation still retains evidence for the victim class, and we define $\eta_{\mathrm{ref}}$ as the fraction of successful baseline attacks for which this direct victim evidence is absent.  
Thus, a small $\eta_{\mathrm{ref}}$ means that victim evidence survives in most successful attacks, although this does not guarantee correct recovery.
When the released pair equals the true attacked pair, every successful attack
with surviving victim evidence is selected by the direct semantic arm.
Therefore, 
\begin{equation}
    \eta_{\mathrm{loc}}\leq\eta_{\mathrm{ref}}\enspace.
    \label{eq:theory_reference_localization}
\end{equation}
Additional residual branches can only reduce the miss rate further.

\textbf{Residual attack success.} \hspace{0.5ex}
Among triggered victim pixels, let $\varepsilon_{\mathrm{rep}}$ be the fraction of localized baseline attacks that still predict $t$ after repair, and let $\delta_{\mathrm{new}}$ denote defended target predictions outside the original
baseline attack event. 
Then the defended attack success rate is bounded by
\begin{equation}
    \ASR_{\mathrm{def}}
    \leq
    \ASR_{\mathrm{base}}
    \bigl[
        \eta_{\mathrm{loc}}
        +(1-\eta_{\mathrm{loc}})\varepsilon_{\mathrm{rep}}
    \bigr]
    +\delta_{\mathrm{new}}\enspace.
    \label{eq:theory_residual_asr}
\end{equation}
The bound separates three failure modes: missed attacked pixels, unsuccessful repair on localized pixels, and new target predictions introduced elsewhere.
When the released pair equals the true attacked pair, combining
$\eta_{\mathrm{loc}}\leq\eta_{\mathrm{ref}}$ with the residual-ASR bound gives
\begin{equation}
    \ASR_{\mathrm{def}}
    \leq
    \ASR_{\mathrm{base}}
    \bigl[
        \eta_{\mathrm{ref}}
        +(1-\eta_{\mathrm{ref}})\varepsilon_{\mathrm{rep}}
    \bigr]
    +\delta_{\mathrm{new}}\enspace.
    \label{eq:theory_reference_asr}
\end{equation}
Thus, surviving internal victim evidence limits the fraction of successful attacks that can be missed by the localizer. 
Whether these localized attacks are actually recovered still depends on the subsequent repair step. 
The No New Target Veto prevents the defense from introducing unsupported target predictions outside the authorized repair support.

\textbf{Attack-unaware pair identification.} \hspace{0.5ex}
The victim and target classes are not assumed to be known. 
Instead, RePAIR evaluates all ordered class pairs using an attack-unaware image-level evidence score.
For a pair $(a,b)$, let $G_i(a,b)$ denote the score on image-level, and let $g_{ab}=\mathbb E[G_i(a,b)]$ be its expected value.
For analysis, we decompose this score into natural evidence and an attack-induced contribution $g_{ab}=\nu_{ab}+\beta_{ab}$. 
Advantage over every competitor exceeds the strongest competing
natural-evidence advantage plus $\delta_{\mathrm{est}}$. 
The attacked pair can be uniquely identified if its attack-induced advantage over every competing pair is larger than both the strongest natural-evidence advantage of a competitor and the estimation error $\delta_{\mathrm{est}}$. 
This provides a sufficient separation condition for pair identification, but does not constitute a distribution-free guarantee for the implemented multi-route selector.

\textbf{Clean-side safety.} \hspace{0.5ex}
Let $\kappa_{\mathrm{cl}}$ denote the probability that a benign pixel is selected for repair, and let $\delta_{\mathrm{spill}}^{\mathrm{cl}}$ denote the probability of a defended prediction change outside the selected mask. 
Then
\begin{equation}
    \Prob_{\mathrm{cl}}(\Ydef\neq\Yhead)
    \leq
    \kappa_{\mathrm{cl}}+\delta_{\mathrm{spill}}^{\mathrm{cl}} \enspace,
    \label{eq:theory_clean_change}
\end{equation}
where $\Prob_{\mathrm{cl}}$ is the distribution over clean inputs and $\Ydef$ the prediction after repair.  
Thus, clean prediction changes are bounded by the probability of selecting benign pixels and the probability of changes outside the repair mask. 
Exact preservation is the special case where both terms are zero. 
The reference-calibration budgets are empirical constraints on the safety-approved population and do not provide population-level guarantees for unseen benign inputs.
The framework does not require reconstruction of an unobserved clean class distribution. 
Instead, it uses surviving victim evidence, attack-unaware pair identification, conservative localization, and a safety-approved reference population to selectively repair target predictions. 
The forward transport and its reversal provide the conceptual model, while the implemented defense approximates this reversal using model-derived evidence. Complete definitions and proofs are provided in Appendix~\ref{app:theory}.

%===================================================

\section{Method}
\label{sec:method}

RePAIR is a post-hoc, label-free, attack-unaware defense for targeted semantic segmentation backdoors. 
The compromised model remains frozen. The defense uses
neither trigger reconstruction nor retraining, and uses no semantic ground truth, clean/poison provenance, attack mask, attack family, poisoning rate, %clean twin, 
or predefined victim-target pair. 
Its inputs are the model and an
unlabeled image pool that may contain triggered samples.
The method has two stages. 
First,  an offline builder identifies a directed victim-target relation from model-derived evidence and constructs a conservative, safety-approved reference state. 
Throughout this section, $(v,t)$ denotes the operational victim--target pair
released by the attack-unaware builder.
At inference, a pair-conditioned localizer identifies inconsistent pixels, feature transport moves them toward victim semantics, and output safety limits unintended changes. 
More details are provided in Appendix~\ref{app:method_details}.

This operational pipeline implements the reverse-transport interpretation of Section~\ref{sec:theory} without estimating the hidden attack operator $T_\xi$, its transported mass $\xi$, or the coupling $\gamma_\xi$. 
Pair discovery estimates the semantic direction, the safety-approved state supplies a victim-side reference geometry, and localization determines where that reverse direction may be applied. 
Repair magnitude is therefore pixel- and branch-dependent rather than an estimate of a single global $\xi$.

\textbf{Architecture-specific semantic reference.} \hspace{0.5ex}
We write the segmentation network as $f=h\circ g$, where $g$ is the feature extractor and $h$ the segmentation head. 
For each pixel $x$, the head produces class probabilities $q_k(x)$ and prediction $\Yhead(x)=\argmax_k q_k(x)$. RePAIR compares this prediction with an architecture-specific internal semantic reference, which differs architectures. 

For convolutional networks (CNNs), we use the normalized final backbone feature $z(x)$. 
For each class $k$, its reference centroid is computed as $\mu_k=|\mathcal P_k|^{-1}\sum_{x\in\mathcal P_k}z(x)$, where $\mathcal P_k$ contains pixels assigned to $k$ by the frozen segmentation head. 
The internal prediction is then the class with the most similar centroid under cosine similarity $\Yproto(x)=\arg\max_k \operatorname{sim}(z(x),\mu_k)$. 
Reference estimation is image-disjoint, so an image is never evaluated using a centroid estimated from that same image.
 
For transformers, semantic evidence is distributed across several transformer blocks rather than concentrated in a single feature space. 
For each feasible victim-target pair $(v,t)$, we therefore use $L$ selected blocks and compute the cosine similarity of each internal feature $z_\ell(x)$ to the corresponding victim and target reference centroids $\mu_{\ell,v}$ and $\mu_{\ell,t}$, obtaining $c_{\ell,v}(x)$ and $c_{\ell,t}(x)$. 
The $2L$ similarities are concatenated into $\phi_{v,t}(x)$ and standardized before a regularized joint logistic readout,
$r_{v,t}(x)=\sigma(w^\top\widetilde\phi_{v,t}(x)+b)$. 
Here, $r_{v,t}(x)$ measures target compatibility, while $p_v^{\mathrm{int}}(x)=1-r_{v,t}(x)$
represents internal victim support. 
The readout is trained only from victim/target pseudo-labels produced by the frozen segmentation head on the reference population, i.e., no semantic or attack labels are used. The joint readout provides the direct victim-evidence signal, while the layer-wise similarities are retained as semantic coordinates for the residual geometry.

\textbf{Attack-unaware pair discovery.} \hspace{0.5ex}
Every feasible ordered pair $(v,t)$ is initially treated as a candidate victim-target relation.
For CNNs, the primary signal is internal-to-head contradiction: evidence for $(v,t)$ increases when pixels retain victim-class evidence in the internal representation while the segmentation head predicts $t$. 
This evidence is aggregated over frozen, image-disjoint reference views and retained only when the support is sufficiently stable and separated from competing pairs.  
The resulting pair defines the victim-to-target direction used by the subsequent localization and repair stages. 
Thus, RePAIR can identify the attack direction directly from model-derived evidence, without requiring prior knowledge of the victim or target class. 

For transformers, candidate structure is first identified using lightweight head, transfer, centroid, and structural cues. Supported pairs are then evaluated using image-disjoint joint-LogReg out-of-fold (OOF) evidence, with a separate OOF state for each pair.
Stable family proposals are formed from structural, mechanism, and family-verification evidence. The frozen rank prior is retained whenever it passes the candidate viability test. 
If the prior family is not viable, another family may be selected only through the conservative rescue path, which requires bilateral leave-one-out support and bilateral structural consistency. 

After the family has been frozen, its two possible victim-target directions are evaluated using FullEdge, FamilyVerify, and eligible Mechanism evidence.
A direction is accepted only when the required direction agreement is unique.
Structural evidence is not a regular direction vote and is used only to corroborate the rescue rule when the Mechanism direction channel is ineligible. 
If family viability or direction evidence is insufficient, the mapping remains inactive. The exact release contract is given in Appendix~\ref{app:sf_mapping_contract}.

\textbf{Safety-approved reference state.} \hspace{0.5ex}
After a victim-target pair has been selected, a correct mapping is not yet sufficient because the unlabeled image pool may itself contain poisoned samples. 
The builder therefore screens candidate reference images using image-disjoint cross-fitting and a victim-safety firewall. 
Images with unstable internal predictions are first rejected, followed by a safety filter that removes images showing severe victim-class collapse or moderate victim degradation accompanied by an unusually strong victim-to-target contradiction. 
These filters are applied sequentially and monotonically: rejected images are not reintroduced to satisfy a desired reference-set size, and thresholds are not relaxed to rescue an insufficient population.

The builder then checks whether the remaining images provide sufficient reliable support and semantic coverage for the selected victim-target pair. 
If not, the repair state remains inactive and no repair is applied at inference time. 
Otherwise, all runtime components, including semantic references, probability statistics, feature-space reference banks, and repair calibrations, are constructed exclusively from the final safety-approved population. 
We use this term rather than clean reference set because the defense has no access to the true provenance of the images and therefore cannot guarantee that they are trigger-free.

The resulting state provides the empirical victim-side reference used by the runtime defense. 
It is not treated as the unknown clean distribution $p^\star$ and is not assumed to be uncontaminated, i.e., it is instead a model-derived reference population that has passed the defense's one-way safety checks.

\textbf{Pair-consistent localization.} \hspace{0.5ex} 
Localization combines direct semantic contradiction with residual tests for pixels whose internal victim evidence is weak or distributed across multiple features.
A shared probability branch compares the observed target-over-victim probability excess $d_p(x)=q_t(x)-q_v(x)$
with a semantically matched reference excess.
Victim and target reference neighborhoods provide matched reference values and a semantic interpolation coordinate, yielding the expected excess $\widehat d_{\mathrm{ref}}(x)$.
RePAIR measures the one-sided residual $D_{\mathrm{prob}}(x) = [d_p(x)-\widehat d_{\mathrm{ref}}(x)]_+$, which is calibrated image-disjointly on the frozen safety-approved population using an empirical reference-pixel budget. 
This budget controls the probability-residual calibration and is distinct from the repair margin. 
Matching is performed after removing the victim-target semantic axis, so the expected margin is conditioned on the remaining semantic geometry rather than on the displacement being tested.

CNNs combine the preserved semantic contradiction with target-manifold, self-counterfactual, and probability-residual cues. 
Transformer use five sequential, mutually exclusive stages restricted to baseline head-target pixels: semantic contradiction, residual victim-manifold rupture, residual
target-manifold rupture, residual counterfactual manifold trajectory, and residual semantic probability. 
Their union defines $M_{\mathrm{TF}}$ and satisfies $M_{\mathrm{TF}}\subseteq\{\Yhead=t\}$. 
The staged construction records why a pixel was selected, prevents later branches from relabeling an earlier decision,
and enables branch-specific repair. 
Exact masks and image-disjoint calibration rules are given in Appendix~\ref{app:localization_repair}.

\textbf{Feature-first repair.} \hspace{0.5ex}
RePAIR first repairs the internal representation of localized pixels before applying any residual output correction. 
For each localized feature, the discovered victim-target direction is removed from the retrieval geometry, and safety-approved victim references are searched in the orthogonal complement. 
The retrieved neighborhood is aggregated and the feature is transported toward the victim reference with branch-dependent strength, while preserving its original norm. 
The modified representation is then passed through the unchanged remainder of the network. 
By matching in the orthogonal complement, the reference search is not dominated by the same victim-target displacement that the repair is intended to reverse.

If the feature repair does not fully restore the victim prediction, RePAIR applies a conservative output correction restricted to the discovered victim-target pair. 
The correction is calibrated from safety-approved victim references and can be partial depending on the repair branch. 
Thus, feature transport and output correction address the same victim-target discrepancy at different stages: feature transport moves the internal representation toward victim semantics, while the output step corrects residual target predictions by directly modifying only the victim and target
logits.

The resulting feature transport approximates the reverse-transport direction without requiring the unknown clean distribution or transport magnitude. 
The discovered pair provides the direction, while the safety-approved victim references provide the local repair target.

\textbf{Output safety and decision isolation.} \hspace{0.5ex}
Let $\Yrep$ denote the prediction after repair but before output safety. 
The No New Target Veto restores $\Yhead$ whenever the repair changes a pixel outside the authorized mask to the discovered target class $t$, provided that the original head prediction was not $t$. 
Thus, the veto prevents unsupported new target predictions outside the repair support while leaving existing target predictions unaffected. 
The veto is intentionally asymmetric and only constrains the creation of the discovered target.

%=====================================================================
\section{Experiments}

\subsection{Experimental Setting}
\label{sec:setup_exp}

We evaluate RePAIR on Cityscapes \citep{cordts2016Cityscapes}, PASCAL VOC 2012 \citep{everingham2015pascal} and ADE20K \citep{zhou2017scene}. 
We consider two representative segmentation architectures, DeepLabV3+ with a
ResNet-50 backbone \citep{Chen2018_deeplab} and SegFormer with a MiT-B5
backbone \citep{xie2021segformersimpleefficientdesign}.

We evaluate backdoor attacks at a $10\%$ poison rate, including IBA with NNI and PRL strategies, ConSeg, OFBA, and BadSeg with all four attack configurations.
For all object-to-background attacks (IBA NNI, IBA PRL, ConSeg, OFBA, and BadSeg-O2B), we use car$\rightarrow$road as the victim-target pair. 
For the remaining attack configurations, we use car$\rightarrow$person for object-to-object, sidewalk$\rightarrow$road for background-to-background, and road$\rightarrow$car for background-to-object attacks. 

We compare RePAIR against all defense methods, i.e., pruning, fine-tuning, Segment Cleanse, and ABL. 
Performance is evaluated using Attack Success Rate (ASR), benign mean intersection-over-union (mIOU-B), and attacked mIoU (mIOU-A). 
For BadSeg-B2O, ASR is restricted to ground-truth victim pixels inside the designated attack region. 
More details on metric definitions and experimental configurations are provided in Appendix~\ref{app:exp_setting}. 

%=====================================================================
\subsection{Numerical Results}
\label{sec:results}

\textbf{Attack success rate.} \hspace{0.5ex}
Table~\ref{tab:cityscapes} reports attack success rate before and after applying the evaluated defenses for the Cityscapes dataset. 
Lower values indicate stronger attack mitigation. ASR measures the fraction of victim-class pixels that are misclassified as the attacker-selected target class on triggered inputs. 

For DeepLabV3+, the undefended model achieves high ASR across all attacks, ranging from $60.83\%$ to $97.41\%$. 
Fine-tuning and pruning provide little mitigation and leave the ASR high for most attacks. 
In contrast, RePAIR consistently reduces the ASR to below $7.86\%$ across all evaluated attacks. 
This consistent performance across different attack types is not observed for the other defense methods. 
Segment Cleanse achieves effective mitigation for IBA and OFBA, but performs substantially worse on the BadSeg attacks compared to our defense. 
Overall, RePAIR achieves the best or second-best ASR for most attacks and maintains strong mitigation consistently across the full attack set. 

For SegFormer, the undefended model shows lower ASR for several attacks, with values ranging from $26.43\%$ to $95.46\%$. 
Fine-tuning and pruning again provide limited mitigation, while Segment Cleanse also fails to consistently reduce the ASR. 
ABL provides more consistent mitigation across the evaluated attacks. 
RePAIR achieves lower ASR than all other evaluated defenses on all backdoor attacks, reducing the ASR to below $0.54\%$ and frequently to zero. 
Thus, for these cases, our defense prevents the attack from successfully redirecting any victim-class pixels to the target class. 

\textbf{Segmentation performance.} \hspace{0.5ex} 
Table~\ref{tab:cityscapes} reports the mean mIoU-B and mIoU-A across all evaluated attacks in the two rightmost columns. 
The clean models, without a backdoor attack or defense, achieve an mIoU-B of $78.74\%$ for DeepLabV3+ and $80.46\%$ for SegFormer, representing the upper reference for benign segmentation performance. 
Under attack, the goal is to preserve this performance on benign classes while recovering the attacked classes. 

The undefended backdoored models retain an mIoU-B of $73.28\%$ for DeepLabV3+ and $75.80\%$ for SegFormer, showing that the attacks primarily affect the targeted classes. 
In contrast, ABL substantially degrades mIoU-B, indicating that its mitigation affects benign predictions beyond the attacked classes. 
RePAIR maintains mIoU-B at $74.50\%$ for DeepLabV3+ and $72.30\%$ for SegFormer, while achieving the highest mIoU-A among the evaluated defenses. 
Thus, RePAIR combines recovery of the attacked classes with preservation of the model's general segmentation performance.\\

\begin{table}[t]
\caption{
\textbf{Cityscapes results across backdoor attacks.} ASR is the primary metric, with lower values indicating better defense performance. mIoU-A and mIoU-B are averaged across the eight attacks and reported with their minimum and maximum deviations from the mean.}
\label{tab:cityscapes}
\centering
\footnotesize
\setlength{\tabcolsep}{4pt}
\renewcommand{\arraystretch}{1.4} % Standard ist 1.0
\resizebox{\textwidth}{!}{%
\begin{tabular}{llcccccccc|cc}
\toprule
Model & Method & IBA (NNI) & IBA (PRL) & ConSeg & OFBA & BadSeg-O2B & BadSeg-O2O & BadSeg-B2O & BadSeg-B2B & mIoU-A & mIoU-B \\
\midrule
& no defense & $88.56$ & $97.41$ & $86.23$ & $60.83$ & $95.49$ & $97.03$ & $96.87$ & $96.82$ & $71.14_{-7.28}^{+6.49}$ & $73.28_{-10.01}^{+6.00}$  \\
& fine-tuning & $56.36$ & $12.35$ & $84.54$ & $6.31$ & $13.42$ & $\underline{81.91}$ & $86.36$ & $97.69$ & $35.86_{-16.03}^{+14.73}$ & $37.65_{-21.43}^{+18.41}$ \\
Deep- & pruning & $89.60$ & $98.13$ & $90.07$ & $71.36$ & $96.63$ & $96.55$ & $97.58$ & $97.25$ & $70.13_{-6.58}^{+6.82}$ & $72.37_{-9.91}^{+6.00}$ \\
LabV3+ & Seg.-Cleanse & $\mathbf{0.36}$ & $\mathbf{0.89}$ & $86.23$ & $\mathbf{0.75}$ & $\underline{3.31}$ & $94.20$ & $96.85$ & $\underline{36.95}$ & $71.14_{-7.36}^{+6.49}$ & $74.14_{-7.31}^{+5.30}$ \\
& ABL & $83.94$ & $94.55$ & $\mathbf{1.12}$ & $\underline{0.78}$ & $17.29$ & $90.36$ & $\underline{86.19}$ & $92.17$ & $36.66_{-6.48}^{+4.39}$ & $40.97_{-4.08}^{+5.46}$ \\
& RePAIR (ours) & $\underline{1.18}$ & $\underline{1.01}$ & $\underline{1.51}$ & $1.10$ & $\mathbf{2.35}$ & $\mathbf{1.00}$ & $\mathbf{0.46}$ & $\mathbf{7.86}$ & $73.83_{-7.77}^{+3.97}$ & $74.50_{-5.69}^{+3.17}$ \\
\midrule
& no defense & $26.43$ & $28.71$ & $89.15$ & $45.32$ & $87.39$ & $86.24$ & $94.11$ & $95.46$ & $71.95_{-21.68}^{+7.83}$ & $75.80_{-27.52}^{+4.49}$ \\
& fine-tuning & $26.57$ & $35.44$ & $93.90$ & $46.92$ & $86.46$ & $85.12$ & $89.73$ & $95.10$ & $69.54_{-25.96}^{+8.77}$ & $73.05_{-33.14}^{+5.87}$ \\
Seg- & pruning & $26.80$ & $29.45$ & $90.06$ & $45.64$ & $87.53$ & $86.22$ & $93.76$ & $95.51$ & $71.74_{-22.43}^{+7.94}$ & $75.54_{-28.54}^{+4.64}$ \\
Former & Seg.-Cleanse & $21.07$ & $22.71$ & $75.46$ & $36.98$ & $80.67$ & $85.30$ & $94.11$ & $92.91$ & $71.91_{-20.72}^{+7.87}$ & $75.77_{-26.49}^{+4.52}$ \\
& ABL & $\underline{10.89}$ & $\underline{15.09}$ & $\underline{0.71}$ & $\underline{9.11}$ & $\underline{62.48}$ & $\underline{2.32}$ & $\underline{73.35}$ & $\underline{17.60}$ & $20.87_{-9.36}^{+9.14}$ & $22.50_{-9.46}^{+8.80}$ \\
& RePAIR (ours) & $\mathbf{0.00}$ & $\mathbf{0.00}$ & $\mathbf{0.54}$ & $\mathbf{0.00}$
 & $\mathbf{0.00}$ & $\mathbf{0.00}$ & $\mathbf{0.00}$ & $\mathbf{0.31}$
 & $74.52_{-21.34}^{+4.78}$
 & $72.30_{-19.16}^{+4.80}$ \\ 
\bottomrule
\end{tabular}
}
\end{table}

\textbf{Other datasets.} \hspace{0.5ex}  
We additionally evaluate RePAIR on PASCAL VOC 2012 and ADE20K, considering only Object-to-Background attack configurations for both datasets. 
On ADE20K, the evaluated attack configurations do not produce a sufficiently strong backdoor effect in the undefended models to support a meaningful mitigation study. This behavior is further analyzed theoretically in Appendix~\ref{app:theory}, and the corresponding results are reported in Table~\ref{tab:ade20k}. 

PASCAL VOC differs from Cityscapes primarily in the number of semantic classes present within each image. Images in PASCAL VOC typically contain a single foreground object class together with a background label that aggregates all other classes. This different class composition affects the learned class representations and, consequently, the information available to the defense. 
Across the evaluated attacks, the undefended models achieve ASR ranging from $12.21\%$ for NNI to $99.71\%$ for ConSeg, while RePAIR reduces the remaining ASR to between $5.41\%$ and $19.67\%$. The corresponding results are reported in Table~\ref{tab:voc_asr} and \ref{tab:voc_miou}.

\section{Conclusion}
\label{sec:conclusion}

We introduced RePAIR, a post-hoc, label-free, and attack-unaware defense that exploits surviving internal victim-class evidence to localize and selectively repair backdoor-affected pixels without retraining the compromised segmentation model. 
Across eight attack configurations on Cityscapes, RePAIR reduces the defended ASR to $0.46\%$-$7.86\%$ for DeepLabV3+ and to at most $0.54\%$ for SegFormer, with complete suppression of the measured attack success in six of the eight SegFormer settings. 
At the same time, RePAIR preserves substantial segmentation utility, achieving an average mIoU-B of $74.50\%$ for DeepLabV3+ and $72.30\%$ for SegFormer, compared with $73.28\%$ and $75.80\%$ for the corresponding undefended models. 

Beyond the overall mitigation performance, the experiments show that effective backdoor repair can be achieved without explicitly reconstructing the trigger or modifying the compromised model. 
The directed-transport formulation provides a view of the attack as a structured victim-to-target shift, while the internal representation provides a model-derived reference for reversing this shift selectively. 
Across convolutional and transformer-based architectures, these results indicate that surviving internal semantic evidence can serve as a basis for attack-unaware localization and constrained pixel-wise repair.

%\clearpage
\subsection*{AI use statement}

In this work, we used generative AI tools to assist coding, debugging,
mathematical-notation consistency, and LaTeX drafting. We did
not use generative AI tools to generate experimental measurements or to replace
human evaluation of the theoretical claims. AI-assisted outputs were reviewed for mathematical consistency, and generated
LaTeX was compiled and visually inspected. We take responsibility for
the final content of this work, including text, claims, and artifacts produced with
the aid of generative AI.

\subsection*{Ethics statement}

This work studies defenses against data-poisoning backdoors in semantic segmentation,
with emphasis on safety-critical perception. Published attacks are used only as
controlled evaluation mechanisms for the proposed defense. The paper does not propose
a new attack or expand the attacker's capabilities. Any release of code, models, or
configuration files will follow the licenses of the underlying datasets and software
and will be framed for defensive research and reproducible evaluation. The study does
not involve human-subject experiments.

\subsection*{Reproducibility statement}
We provide the implementation of RePAIR, including the model-specific feature extraction, victim-target pair selection, reference construction, localization, feature transport, and output-safety procedures. We use the publicly available Cityscapes, PASCAL VOC 2012, and ADE20K datasets and evaluate the method on DeepLabV3+ with ResNet-50 and SegFormer with MiT-B5. The attack and defense configurations, evaluation metrics, and implementation details are described in the main paper and the appendix. %The code will be uploaded on Github after acceptance.

\bibliographystyle{conference}
\bibliography{biblio}

\clearpage
\appendix
\section{Appendix}

\subsection{Theory: Full Statements and Proofs}
\label{app:theory}

In the following, we formalize the local transport operator, distinguish its inverse from the transposed transport coupling, and describe its operational interpretation for RePAIR. 
The emphasis here is on the formal domain, feasibility conditions, and proof obligations underlying the theoretical claims.

%%%
\subsubsection{Probability space and standing conventions}
\label{app:space}

Let $\Omega$ denote the evaluated pixel population and
\[
    Y:\Omega\rightarrow[K]\cup\{\varnothing\} \enspace,
    \qquad
    \Yhead,\Yrep,\Ydef:\Omega\rightarrow[K] \enspace,
\]
where $K$ is the ground-truth class, $\Yhead$ the undefended head prediction, $\Yrep$ the repaired
prediction before output safety, and $\Ydef$ the final defended prediction.
Pixels with $Y=\varnothing$ are excluded from semantic metrics.

Let $T$ denote the event that a trigger is present, and assume $\Prob(T)>0$ and
$\Prob(T^c)>0$. Define
\begin{equation}
    \Prob_{\mathrm{trig}}(\cdot)=\Prob(\cdot\mid T) 
\end{equation}
as the conditional distribution over triggered inputs, and  
\begin{equation}
    \Prob_{\mathrm{cl}}(\cdot)=\Prob(\cdot\mid T^c)
\end{equation}
as the conditional distribution over clean inputs.

For victim class $v$ and target class $t\neq v$, assume
$\Prob_{\mathrm{trig}}(Y=v)>0$ and define
\begin{equation}
    \Prob_v(B)
    :=
    \Prob_{\mathrm{trig}}(B\mid Y=v) 
    \label{eq:app_victim_measure}
\end{equation}
as the conditional probability over triggered pixels whose ground-truth class is $v$.
Let
\begin{equation}
    A:=\{\Yhead=t\}
\end{equation}
denote the event that the undefended model predicts the target class. 
We then define the baseline and defended attack success rates as 
\begin{equation}
    \ASR_{\mathrm{base}}=\Prob_v(A) \enspace,
    \qquad
    \ASR_{\mathrm{def}}=\Prob_v(\Ydef=t) \enspace.
    \label{eq:app_asr}
\end{equation}

For theoretical evaluation, $(v,t)$ denotes the true attacked victim-target
pair, which unknown to the defense. 
The attack-unaware builder either
releases an active runtime state with discovered pair
$\widehat P=(\hat v,\hat t)$ or abstains and releases an inactive fail-closed
state. 
Let $R\in\{0,1\}$ indicate whether an active state is released. 
We use the following
conditions throughout the analysis:
\begin{align}
    C_{\mathrm{pair}}
    &:\quad R=1 \text{ and } (\hat v,\hat t)=(v,t) \enspace, \\
    C_{\mathrm{tgt}}
    &:\quad R=1 \text{ and } \hat t=t \enspace.
    \label{eq:app_mapping_conditions}
\end{align}
Clearly, $C_{\mathrm{pair}}$ implies $C_{\mathrm{tgt}}$. 
All pair-conditioned
runtime quantities are constructed from the released state and therefore use
$(\hat v,\hat t)$ rather than the unknown pair $(v,t)$. 
The residual-ASR decomposition below applies to the actually executed defense state without assuming $C_{\mathrm{pair}}$. 
In contrast, statements that relate localization error to internal victim evidence for the true pair require $C_{\mathrm{pair}}$. 
Guarantees concerning the creation of the true target class require only $C_{\mathrm{tgt}}$. 
If $R=0$, fail-closed inference leaves the undefended
prediction unchanged. 

Let $\rho_v\in\{0,1\}$ denote the evaluation-side internal victim-evidence event for the true pair.
For CNNs,
$\rho_v=\indic\{\Yproto=v\}$. For transformers, $\rho_v = \indic\{p_v^{\mathrm{int}}>\tau_{\mathrm{sem}}\}$, where $\tau_{\mathrm{sem}}$ is the frozen, pair-specific clean-null threshold for internal victim support. 
In the official transformer state it is calibrated as the $0.995$ empirical quantile of image-disjoint out-of-fold victim-support scores on the final safety-approved reference population.
For an active released state, the implemented direct semantic localization arm
uses the discovered pair,
\begin{equation}
    M_{\mathrm{semantic}}
    =
    \indic\{\Yhead=\hat t\}\rho_{\hat v} \enspace.
    \label{eq:app_semantic_localizer}
\end{equation}
Here $\rho_{\hat v}$ is constructed by the same architecture-specific rule using the released pair.
Let $M:\Omega\rightarrow\{0,1\}$ denote the final localization indicator. 
The implemented masks preserve the semantic arm additively, so
\begin{equation}
    M_{\mathrm{semantic}}\leq M \enspace.
    \label{eq:app_localizer_inclusion}
\end{equation}
Under $C_{\mathrm{pair}}$, the runtime quantities satisfy
$\rho_{\hat v}=\rho_v$ and
$M_{\mathrm{semantic}}=\indic\{\Yhead=t\}\rho_v$. 
If $R=0$, we set $M_{\mathrm{semantic}}=M=0$ consistently with fail-closed inference.

To make all statements well-defined at zero prevalence, we define
\begin{equation}
\eta_{\mathrm{ref}}
=
\begin{cases}
\Prob_v(\rho_v=0\mid A), & \Prob_v(A)>0 \enspace,\\
0, & \Prob_v(A)=0\enspace,
\end{cases}
\label{eq:app_eta_ref}
\end{equation}
\begin{equation}
\eta_{\mathrm{loc}}
=
\begin{cases}
\Prob_v(M=0\mid A), & \Prob_v(A)>0\enspace,\\
0, & \Prob_v(A)=0\enspace,
\end{cases}
\label{eq:app_eta_loc}
\end{equation}
and
\begin{equation}
\varepsilon_{\mathrm{rep}}
=
\begin{cases}
\Prob_v(\Ydef=t\mid M=1,A), & \Prob_v(M=1,A)>0\enspace,\\
0, & \Prob_v(M=1,A)=0\enspace.
\end{cases}
\label{eq:app_epsilon_rep}
\end{equation}
Finally,
\begin{equation}
    \delta_{\mathrm{new}}
    =\Prob_v(\Ydef=t,A^c) \enspace.
    \label{eq:app_delta_new}
\end{equation}
These zero-prevalence conventions change no positive-mass conditional and make
all products in the subsequent bounds unambiguous.

%%%
\subsubsection{Fixed-simplex geometry}
\label{app:simplex}

The class-probability simplex and its tangent space are
\begin{align}
    \Delta^{K-1}
    &=\{p\in\mathbb R_{\geq0}^{K}:\mathbf1^\top p=1\} \enspace,\\
    \mathcal T_\Delta
    &=\{d\in\mathbb R^K:\mathbf1^\top d=0\} \enspace.
\end{align}
For pair $(v,t)$ define
\begin{equation}
    d_{vt}=e_t-e_v\in\mathcal T_\Delta \enspace,
    \label{eq:app_direction}
\end{equation}
where $e_c$ denotes the one-hot basis vector for class $c$.

\begin{lemma}[Directed victim--target displacement]
\label{lem:app_simplex}
Let $p^\star\in\Delta^{K-1}$ and let
\begin{equation}
    q=p^\star+\xi d_{vt} \enspace,
    \qquad 0\leq\xi\leq p_v^\star \enspace,
    \label{eq:app_attack_line}
\end{equation}
with $q\in\Delta^{K-1}$.
Here, $\xi$ denotes the amount of probability mass transported from the victim coordinate $v$ to the target coordinate $t$, and $d_{vt}=e_t-e_v$ is the corresponding directed victim-to-target displacement.
The constraint $0\leq\xi\leq p_v^\star$ ensures that the victim coordinate remains nonnegative.
For
$\psi_{vt}(p)=p_v-p_t$ \enspace,
\begin{equation}
    \psi_{vt}(q)=\psi_{vt}(p^\star)-2\xi \enspace.
    \label{eq:app_margin_shift}
\end{equation}
Moreover, with the sign convention positive on the victim side, the signed Euclidean distance to
$\mathcal H_{vt}=\{p:p_v=p_t\}$ is
\begin{equation}
    \operatorname{sdist}(p,\mathcal H_{vt})
    =\frac{p_v-p_t}{\sqrt2} \enspace.
    \label{eq:app_boundary_distance}
\end{equation}
\end{lemma}

\begin{proof}
Equation~\ref{eq:app_attack_line} gives
$q_v=p_v^\star-\xi$ and $q_t=p_t^\star+\xi$, proving the margin identity.
The hyperplane normal is $e_v-e_t$, whose Euclidean norm is $\sqrt2$. 
Thus, the stated signed-distance formula follows.
\end{proof}

\begin{corollary}[Pairwise preference flip]
\label{cor:app_pair_flip}
If $p_v^\star>p_t^\star$, then the strict pairwise ordering $q_t>q_v$ holds whenever
\begin{equation}
    \xi>\frac{p_v^\star-p_t^\star}{2} \enspace.
    \label{eq:app_pair_flip}
\end{equation}
At equality the pair is tied.
\end{corollary}

The corollary concerns only the victim-target pairwise boundary. 
A complete change of $\arg\max$ additionally requires $q_t$ to exceed every other class coordinate. 
The construction is a local analytical model, i.e., RePAIR does not observe $p^\star$ at runtime.

%%%
\subsubsection{Formal transport operator and coupling}
\label{app:transport}

Section~\ref{sec:theory} introduces the local operator
$T_\xi(p)=p+\xi(e_t-e_v)$. 
Here we record its feasible inverse domain and the coupling construction used in the proofs.

Define
\begin{equation}
    \mathcal D_\xi
    =
    \{p\in\Delta^{K-1}:p_v\ge\xi\} \enspace,
    \qquad
    \mathcal R_\xi
    =
    \{q\in\Delta^{K-1}:q_t\ge\xi\} \enspace.
    \label{eq:app_transport_domains}
\end{equation}

\begin{lemma}[Local inverse operator]
\label{lem:app_inverse_operator}
For fixed $\xi\ge0$, the map
\begin{equation}
    T_\xi:\mathcal D_\xi\rightarrow\mathcal R_\xi \enspace,
    \qquad
    T_\xi(p)=p+\xi(e_t-e_v) \enspace,
\end{equation}
is bijective with
\begin{equation}
    T_\xi^{-1}(q)
    =
    q+\xi(e_v-e_t) \enspace.
    \label{eq:app_inverse_operator}
\end{equation}
\end{lemma}

\begin{proof}
For $p\in\mathcal D_\xi$, the victim coordinate remains nonnegative after subtracting $\xi$, while all coordinates still sum to one, so $T_\xi(p)\in\Delta^{K-1}$ and its target coordinate is at least $\xi$.
Conversely, for $q\in\mathcal R_\xi$, subtracting $\xi$ from $q_t$ and adding it to $q_v$ yields a valid simplex point in $\mathcal D_\xi$.
Direct substitution shows that the two maps are mutual inverses.
\end{proof}

\begin{lemma}[One-edge coupling]
\label{lem:transport}
Let $p\in\mathcal D_\xi$ and $q=T_\xi(p)$. Define
$\gamma_\xi\in\mathbb R_{\ge0}^{K\times K}$ by
\[
    (\gamma_\xi)_{vv}=p_v-\xi \enspace,
    \qquad
    (\gamma_\xi)_{vt}=\xi \enspace,
    \qquad
    (\gamma_\xi)_{kk}=p_k\quad(k\neq v) \enspace,
\]
with all remaining entries zero.
Then
\begin{equation}
    \gamma_\xi\mathbf1=p \enspace,
    \qquad
    \gamma_\xi^\top\mathbf1=q \enspace.
    \label{eq:app_coupling_marginals}
\end{equation}
Hence $\gamma_\xi$ is a feasible coupling between the forward marginals $p$ and $q$, while $\gamma_\xi^\top$ is a feasible coupling from $q$ back to $p$.
\end{lemma}

\begin{proof}
The victim row sums to $(p_v-\xi)+\xi=p_v$ and every other row to $p_k$.
The victim column sums to $p_v-\xi=q_v$, the target column to
$p_t+\xi=q_t$, and every other column to $p_k=q_k$. Transposition exchanges these marginals.
\end{proof}

The last statement concerns a joint measure, not a function on probability vectors. 
Thus $\gamma_\xi^\top$ represents reversed mass flow but is not the operator $T_\xi^{-1}$.
No optimality claim is made for either coupling without specifying a transport cost.

%%%
\subsubsection{Internal-reference survival and localization}
\label{app:reference}

\begin{lemma}[Surviving victim evidence]
\label{lem:app_reference_mass}
Under the conventions in Section~\ref{app:space},
\begin{equation}
    \Prob_v(A,\rho_v=1)
    =
    (1-\eta_{\mathrm{ref}})\ASR_{\mathrm{base}} \enspace.
    \label{eq:app_reference_mass}
\end{equation}
\end{lemma}

\begin{proof}
If $\Prob_v(A)>0$, the multiplication rule gives
$\Prob_v(A,\rho_v=1)=\Prob_v(\rho_v=1\mid A)\Prob_v(A)$, which is the claimed expression. 
If $\Prob_v(A)=0$, both sides are zero by definition.
\end{proof}

The relevant condition is therefore not complete encoder invariance.
$\eta_{\mathrm{ref}}$ measures only the fraction of successful baseline attacks for which the direct internal victim-evidence event is absent.

\begin{proposition}[Internal evidence bounds localization misses]
\label{prop:reference_to_localization}
Assume $C_{\mathrm{pair}}$ and
Equation~\ref{eq:app_localizer_inclusion}. Then
\begin{equation}
    \eta_{\mathrm{loc}}\leq\eta_{\mathrm{ref}} \enspace.
    \label{eq:app_reference_localization_bound}
\end{equation}
\end{proposition}

\begin{proof}
Under $C_{\mathrm{pair}}$, the released pair satisfies
$(\hat v,\hat t)=(v,t)$, so the implemented semantic arm is exactly $M_{\mathrm{semantic}}=\indic\{\Yhead=t\}\rho_v$.
If $\Prob_v(A)=0$, both quantities are zero by convention. Otherwise, on $A$, $\rho_v=1$ implies $M_{\mathrm{semantic}}=1$, and Equation~\ref{eq:app_localizer_inclusion} implies $M=1$. 
Therefore $\{M=0\}\cap A\subseteq\{\rho_v=0\}\cap A$. Conditioning on $A$ proves the inequality.
\end{proof}

Additional residual branches may reduce localization misses beyond the direct semantic arm but cannot invalidate this inclusion.

%%%
\subsubsection{Directed pair evidence and population separation}
\label{app:separability}

Pair discovery is architecture specific, so we first state a generic scalar separation result.
For each ordered pair $(a,b)$, let $G_i(a,b)$ be an
attack-unaware image-level evidence statistic and
\begin{equation}
    g_{ab}:=\mathbb E[G_i(a,b)] \enspace.
\end{equation}
Given $n$ builder images for which the statistic is defined, let
\begin{equation}
    \widehat g_{ab}
    :=
    \frac{1}{n}\sum_{i=1}^{n}G_i(a,b)
\end{equation}
denote the corresponding empirical statistic.
For analysis only, let $\nu_{ab}$ denote the expected value of the same statistic under a corresponding no-backdoor reference population and define
\begin{equation}
    \beta_{ab}:=g_{ab}-\nu_{ab} \enspace,
    \qquad
    g_{ab}=\nu_{ab}+\beta_{ab} \enspace.
    \label{eq:app_flow_decomposition}
\end{equation}
Neither $\nu_{ab}$ nor $\beta_{ab}$ is observed by the defense.

For CNNs, one primitive statistic is
\begin{equation}
    G_i^{\mathrm{CNN}}(a,b)
    =
    \frac{1}{|\Omega_i|}
    \sum_{x\in\Omega_i}
    \indic\{\Yproto(x)=a\wedge\Yhead(x)=b\} \enspace.
\end{equation}
For transformers, one primitive OOF statistic is a fixed upper-quantile summary of internal victim support on head-target pixels,
\begin{equation}
    G_i^{\mathrm{TF}}(a,b)
    =
    Q_\tau\!
    \left(
        p_a^{\mathrm{int}}(x):\Yhead(x)=b
    \right),
    \label{eq:app_sf_pair_population_score}
\end{equation}
when the eligible set is nonempty. Pairs for which this primitive statistic is undefined are outside the scope of the scalar proposition below.
The full implementation also uses support, structural, and held-out mechanism evidence.

\begin{proposition}[Sufficient scalar pair separation]
\label{prop:app_pair_separation}
Suppose $(v,t)$ is the true attacked pair and, for every competitor
$(a,b)\neq(v,t)$,
\begin{equation}
    \left|
        (\widehat g_{vt}-\widehat g_{ab})
        -(g_{vt}-g_{ab})
    \right|
    \leq\delta_{\mathrm{est}} \enspace.
    \label{eq:app_pair_error_assumption}
\end{equation}
If
\begin{equation}
    \beta_{vt}
    -\max_{(a,b)\neq(v,t)}\beta_{ab}
    >
    \max_{(a,b)\neq(v,t)}(\nu_{ab}-\nu_{vt})
    +\delta_{\mathrm{est}} \enspace,
    \label{eq:app_pair_separation}
\end{equation}
then $\widehat g_{vt}>\widehat g_{ab}$ for every competitor.
\end{proposition}

\begin{proof}
For any competitor,
\[
    g_{vt}-g_{ab}
    =
    (\beta_{vt}-\beta_{ab})-(\nu_{ab}-\nu_{vt}) \enspace.
\]
The stated condition implies $g_{vt}-g_{ab}>\delta_{\mathrm{est}}$ for every
competitor. Equation~\ref{eq:app_pair_error_assumption} then gives $\widehat g_{vt}-\widehat g_{ab}>0$.
\end{proof}

The result is sufficient, not necessary, and is conditional on the stated uniform score-difference error bound. It should not be read as a finite-sample guarantee for the complete transformer mapping selector, whose family and direction decisions use multiple evidence views rather than a single scalar pair ranking.

%%%
\subsubsection{Transformer mapping release contract}
\label{app:sf_mapping_contract}

Transformer mapping is modeled as a partial release rule over attack-unaware
builder evidence.
Let $\mathcal F$ denote the set of feasible undirected class
families. Pair-owned out-of-fold evidence provides structural, mechanism, and family verification views for each candidate family.
These views define a stable rank prior together with a proposal set
\begin{equation}
\mathcal F_{\mathrm{prop}}
=
\mathcal C_{S}
\cup
\mathcal C_{M}
\cup
\mathcal C_{V} \enspace,
\label{eq:app_sf_family_proposals}
\end{equation}
where $\mathcal C_{S}$, $\mathcal C_{M}$, and $\mathcal C_{V}$ are the stable candidate cores induced by the corresponding evidence views.

If $\mathcal F_{\mathrm{prop}}=\varnothing$, the builder releases no mapping and returns $(R,\widehat P)=(0,\bot)$. Otherwise, let $F_{\mathrm{prior}}\in\mathcal F_{\mathrm{prop}}$ denote the leading family
under the frozen stable rank prior.
Let \begin{equation}
\mathsf V_{\mathrm{base}}(F)\in\{0,1\}
\end{equation}
denote the frozen candidate viability test. 
The builder preserves $F_{\mathrm{prior}}$ whenever
$\mathsf V_{\mathrm{base}}(F_{\mathrm{prior}})=1$.

If the prior family is not viable, a different family can be considered only through the rescue path.
Let
\begin{equation}
\mathsf V_{\mathrm{res}}(F)\in\{0,1\}
\end{equation}
indicate that $F$ has the required bilateral leave-one-out support together with bilateral structural consistency. 
The frozen family is therefore
\begin{equation}
\widehat F
=
\begin{cases}
F_{\mathrm{prior}},
&
\mathsf V_{\mathrm{base}}(F_{\mathrm{prior}})=1 \enspace,
\\[1ex]
\mathsf{Select}_{\mathrm{rank}}
\left(
\left\{
F\in\mathcal F_{\mathrm{prop}}
:
\mathsf V_{\mathrm{res}}(F)=1
\right\}
\right) ,
&
\mathsf V_{\mathrm{base}}(F_{\mathrm{prior}})=0
\text{ and the set is nonempty} \enspace,
\\[1ex]
\bot,
&
\text{otherwise} \enspace.
\end{cases}
\label{eq:app_sf_family_release}
\end{equation}
Here $\mathsf{Select}_{\mathrm{rank}}$ denotes selection according to the frozen family ranking. 
No family is released when the required viability evidence is absent.

Assume now that $\widehat F\neq\bot$. The two orientations of
$\widehat F$ form the directed candidate set $\mathcal D(\widehat F)$.
For each
$e\in\mathcal D(\widehat F)$, let
\begin{equation}
D_{\mathrm{FE}}(e) \enspace,
\quad
D_{\mathrm{FV}}(e) \enspace,
\quad
D_{\mathrm{M}}(e)
\in\{0,1\}
\end{equation}
denote the FullEdge, FamilyVerify, and Mechanism direction votes.
Let $E_{\mathrm{M}}\in\{0,1\}$ indicate whether the Mechanism direction evidence
passes its frozen response and endpoint balance requirements.

The regular direction score is
\begin{equation}
V_{\mathrm{dir}}(e)
=
D_{\mathrm{FE}}(e)
+
D_{\mathrm{FV}}(e)
+
E_{\mathrm{M}}D_{\mathrm{M}}(e) \enspace.
\label{eq:app_sf_direction_votes}
\end{equation}
A regular direction is released only when one orientation
$e^\star\in\mathcal D(\widehat F)$ is the unique orientation satisfying
\begin{equation}
V_{\mathrm{dir}}(e^\star)\geq 2 \enspace.
\label{eq:app_sf_direction_release}
\end{equation}

A separate rescue rule is available only when the Mechanism direction channel
is ineligible. 
Let $D_{\mathrm{S}}(e)\in\{0,1\}$ denote structural corroboration. Structural
evidence is not a regular direction vote.
If $E_{\mathrm{M}}=0$, FullEdge and FamilyVerify disagree, and structural evidence
corroborates the FamilyVerify orientation, the FamilyVerify orientation may be
released through the rescue rule provided that the frozen family itself
satisfies the rescue viability condition.

Let $e^\star$ denote the unique directed orientation accepted by either the
regular rule or the rescue rule.
The released transformer mapping state is
\begin{equation}
(R,\widehat P)
=
\begin{cases}
(1,e^\star),
&
\widehat F\neq\bot
\text{ and a unique direction is accepted} \enspace,
\\[1ex]
(0,\bot),
&
\text{otherwise} \enspace.
\end{cases}
\label{eq:app_sf_mapping_release}
\end{equation}

\begin{proposition}[Transformer fail-closed mapping contract]
\label{prop:app_sf_mapping_contract}
The transformer builder releases a directed pair only if a viable family is
frozen and one of its two orientations satisfies the regular direction rule or
the rescue rule above.
If family viability is insufficient, if direction evidence is insufficient, or if the accepted orientation is not unique, then $R=0$ and $\widehat P=\bot$.
Once the direction firewall has rejected the frozen family, no later family fallback can create an active mapping.
\end{proposition}

\begin{proof}
Equation~\ref{eq:app_sf_family_release} returns either a viable frozen family
or $\bot$. Conditional on a viable family, Equation~\ref{eq:app_sf_mapping_release} sets $R=1$ only when exactly one orientation satisfies an authorized direction rule. 
Every other case returns $(R,\widehat P)=(0,\bot)$. Structural evidence can corroborate the rescue rule
but does not contribute to the regular direction vote in Equation~\ref{eq:app_sf_direction_votes}. 
The absence of a later family fallback completes the fail-closed alternative.
\end{proof}

This proposition is a deterministic release contract for the observable
builder state. 
It does not claim that the released pair equals the unknown true attacked pair. 
Ground truth, attack provenance, and the oracle attacked pair are not inputs to the release rule.

%%%
\subsubsection{Repair-state safety}
\label{app:safety}

Let $\mathcal C_0$ be the provisional reference population after architecture-specific cross-fitting and let $V_i\in\{0,1\}$ denote the victim-safety veto.
Define
\begin{equation}
    \mathcal C_{\mathrm{safe}}
    =
    \{i\in\mathcal C_0:V_i=0\} \enspace.
    \label{eq:app_safe_set}
\end{equation}

\begin{proposition}[One-way safety]
\label{prop:app_one_way}
Every later capacity-controlled repair bank satisfies
\begin{equation}
    \mathcal C_{\mathrm{bank}}
    \subseteq
    \mathcal C_{\mathrm{safe}}
    \subseteq
    \mathcal C_0 \enspace.
    \label{eq:app_one_way}
\end{equation}
Consequently, a vetoed reference cannot be reintroduced after the safety decision.
\end{proposition}

\begin{proof}
The first inclusion is the construction rule for all later banks and fixed points. 
The second follows directly from the definition of
$\mathcal C_{\mathrm{safe}}$.
\end{proof}

This guarantees monotone selection, not true clean provenance.
If required reference count or semantic coverage cannot be maintained, the method abstains.

%%%
\subsubsection{Orthogonal semantic matching}
\label{app:feature_geometry}

Let $u(x)$ be the architecture-specific semantic coordinate used by the
probability reference. 
Assume distinct victim and target centers $\mu_v^u\neq\mu_t^u$ and define
\begin{equation}
    a_{vt}
    =
    \frac{\mu_t^u-\mu_v^u}{\|\mu_t^u-\mu_v^u\|_2} \enspace,
    \qquad
    \Pi_{vt}^{\perp}=I-a_{vt}a_{vt}^\top \enspace.
\end{equation}

\begin{lemma}[Orthogonal matching ignores pair-axis displacement]
\label{lem:app_orthogonal}
For semantic coordinates $z,r$ and scalar $c$,
\begin{equation}
    \Pi_{vt}^{\perp}[(z+ca_{vt})-r]
    =
    \Pi_{vt}^{\perp}(z-r) \enspace.
    \label{eq:app_orthogonal_invariance}
\end{equation}
\end{lemma}

\begin{proof}
Since $\|a_{vt}\|_2=1$,
$\Pi_{vt}^{\perp}a_{vt}=(I-a_{vt}a_{vt}^\top)a_{vt}=0$.
\end{proof}

Thus orthogonal retrieval is exactly invariant to displacement along the
chosen victim--target axis.
It does not assume that a real attack lies entirely on that axis.

%%%
\subsubsection{Minimal victim-target logit correction.}
\label{app:projection}

Consider a residual pair-logit substep applied after any feature transport.
For pre-correction logits $\ell\in\mathbb R^K$, let
$q=\operatorname{softmax}(\ell)$ and define
\begin{equation}
    m_\ell
    =
    \ell_v-\ell_t
    =
    \log\frac{q_v}{q_t} \enspace.
\end{equation}
Let $m_{vt}^{\mathrm{rep}}>0$ be the frozen required victim-over-target margin, calibrated from safety-approved victim reference pixels.

\begin{proposition}[Minimum-$\ell_2$ full pair correction]
\label{prop:minimal_pair_projection}
Among corrections $\Delta\in\mathbb R^K$ satisfying
\begin{align}
    \Delta_k&=0 \enspace, \quad (k\notin\{v,t\})\enspace,\\
    \Delta_v+\Delta_t&=0\enspace,\\
    (\ell_v+\Delta_v)-(\ell_t+\Delta_t)
    &\geq m_{vt}^{\mathrm{rep}} \enspace,
\end{align}
the unique minimum-$\ell_2$ correction is
\begin{equation}
    \Delta_v=\frac d2 \enspace,
    \qquad
    \Delta_t=-\frac d2 \enspace,
    \qquad
    d=[m_{vt}^{\mathrm{rep}}-(\ell_v-\ell_t)]_+ \enspace.
    \label{eq:app_minimal_projection}
\end{equation}
\end{proposition}

\begin{proof}
Write $r=m_{vt}^{\mathrm{rep}}-(\ell_v-\ell_t)$. The zero-sum constraint gives
$\Delta_t=-\Delta_v$, so feasibility is $2\Delta_v\geq r$ and $\|\Delta\|_2^2=2\Delta_v^2$.
If $r\leq0$, $\Delta_v=0$ is feasible and uniquely minimizes the norm. If $r>0$, the smallest feasible value is $\Delta_v=r/2$.
Combining both cases yields Equation~\ref{eq:app_minimal_projection}.
\end{proof}

\begin{corollary}[Effect of the full pair correction on odds]
Let $q$ and $q'$ be the softmax distributions immediately before and after the full correction in Proposition~\ref{prop:minimal_pair_projection}.
Then
\begin{equation}
    \frac{q'_v/q'_t}{q_v/q_t}=e^d \enspace.
    \label{eq:app_odds_change}
\end{equation}
\end{corollary}

The actual runtime may use a branch-dependent fraction of this residual correction and may subsequently apply a fail-closed projection. 
Feature transport can also change logits of classes outside $\{v,t\}$ indirectly, i.e., the ``pair-only'' statement applies only to this explicit residual-logit substep.

%%%
\subsubsection{Residual attack success}
\label{app:residual}

\begin{proposition}[Residual-ASR bound]
\label{prop:residual_asr}
Under the definitions in Section~\ref{app:space},
\begin{equation}
    \ASR_{\mathrm{def}}
    \leq
    \ASR_{\mathrm{base}}
    \left[
        \eta_{\mathrm{loc}}
        +(1-\eta_{\mathrm{loc}})\varepsilon_{\mathrm{rep}}
    \right]
    +\delta_{\mathrm{new}} \enspace.
    \label{eq:app_residual_asr}
\end{equation}
\end{proposition}

\begin{proof}
Partition $\{\Ydef=t\}$ into $A$ and $A^c$. The $A^c$ contribution is exactly $\delta_{\mathrm{new}}$.
If $\Prob_v(A)=0$, the $A$ contribution is zero and the claim follows. Otherwise, partition $A$ by $M$. On $A\cap\{M=0\}$, the defended-target event has probability at most $\ASR_{\mathrm{base}}\eta_{\mathrm{loc}}$. 
On $A\cap\{M=1\}$, the exact contribution is $\ASR_{\mathrm{base}}(1-\eta_{\mathrm{loc}})\varepsilon_{\mathrm{rep}}$, including the zero-mass convention of Equation~\ref{eq:app_epsilon_rep}.
Adding the disjoint contributions proves the bound.
\end{proof}

\begin{corollary}[Reference-aware residual-ASR bound]
\label{cor:reference_asr}
Assume $C_{\mathrm{pair}}$ and
Equation~\ref{eq:app_localizer_inclusion}.
Then
\begin{equation}
    \ASR_{\mathrm{def}}
    \leq
    \ASR_{\mathrm{base}}
    \left[
        \eta_{\mathrm{ref}}
        +(1-\eta_{\mathrm{ref}})\varepsilon_{\mathrm{rep}}
    \right]
    +\delta_{\mathrm{new}} \enspace.
    \label{eq:app_reference_asr}
\end{equation}
\end{corollary}

\begin{proof}
Proposition~\ref{prop:reference_to_localization} gives $\eta_{\mathrm{loc}}\leq\eta_{\mathrm{ref}}$. 
For fixed $\varepsilon_{\mathrm{rep}}\in[0,1]$, the function $f(\eta)=\eta+(1-\eta)\varepsilon_{\mathrm{rep}}$ is nondecreasing because $f'(\eta)=1-\varepsilon_{\mathrm{rep}}\geq0$. Substitute this inequality into Proposition~\ref{prop:residual_asr}.
\end{proof}

%%%
\subsubsection{Effect of the No New Target Veto}
\label{app:new_target}

Let $\mathcal S=\{x:M(x)=1\}$ denote the supported repair region.
The No New Target Veto is applied to the discovered target $\hat t$ and enforces
\begin{equation}
    x\notin\mathcal S \enspace,
    \quad
    \Yhead(x)\neq\hat t
    \quad\Longrightarrow\quad
    \Ydef(x)\neq\hat t \enspace.
    \label{eq:app_nnt_property}
\end{equation}

\begin{proposition}[New-target support restriction]
\label{prop:new_target_support}
Assume $C_{\mathrm{tgt}}$, i.e., an active state is released and the discovered target satisfies $\hat t=t$. 
Within $\Prob_v$,
\begin{equation}
    \delta_{\mathrm{new}}
    \leq
    \Prob_v(x\in\mathcal S,A^c) \enspace.
    \label{eq:app_new_target_bound}
\end{equation}
If $\mathcal S\subseteq A$, then $\delta_{\mathrm{new}}=0$.
\end{proposition}

\begin{proof}
Under $C_{\mathrm{tgt}}$, $\hat t=t$, so Equation~\ref{eq:app_nnt_property} excludes $\{\Ydef=t,A^c\}$ outside $\mathcal S$.
Hence any event counted by $\delta_{\mathrm{new}}$ must lie in $\mathcal S\cap A^c$.
If $\mathcal S\subseteq A$, this set is empty.
\end{proof}

\begin{corollary}[No new true-target attacks under target-correct transformer localization]
\label{cor:app_sf_no_new_target}
Assume $C_{\mathrm{tgt}}$. Then the transformer runtime satisfies

\begin{equation*}
    M_{\mathrm{TF}} \subseteq \{\Yhead=\hat t\}= \{\Yhead=t\} = A \enspace.
\end{equation*}

Consequently,
\begin{equation}
\delta_{\mathrm{new}}=0 \enspace.
\end{equation}
\end{corollary}

\begin{proof}
Every transformer localization branch is restricted to baseline pixels predicted as the discovered target $\hat t$, and the final localization mask is their union. 
Hence
$M_{\mathrm{TF}}\subseteq\{\Yhead=\hat t\}$.
Under $C_{\mathrm{tgt}}$ we have $\hat t=t$, so $M_{\mathrm{TF}}\subseteq A$.
Proposition~\ref{prop:new_target_support} therefore gives $\delta_{\mathrm{new}}=0$.
\end{proof}

%%%
\subsubsection{Clean-side intervention}
\label{app:clean}

Define
\begin{equation}
    \kappa_{\mathrm{cl}}=\Prob_{\mathrm{cl}}(M=1)
    \label{eq:app_clean_activation}
\end{equation}
and
\begin{equation}
    \delta_{\mathrm{spill}}^{\mathrm{cl}}
    =
    \Prob_{\mathrm{cl}}(\Ydef\neq\Yhead,M=0) \enspace.
    \label{eq:app_clean_spill}
\end{equation}

\begin{proposition}[Clean prediction-change bound]
\label{prop:clean_change}
\begin{equation}
    \Prob_{\mathrm{cl}}(\Ydef\neq\Yhead)
    \leq
    \kappa_{\mathrm{cl}}
    +\delta_{\mathrm{spill}}^{\mathrm{cl}} \enspace.
    \label{eq:app_clean_change}
\end{equation}
\end{proposition}

\begin{proof}
Partition the prediction-change event by $M$. The localized contribution is at most $\Prob_{\mathrm{cl}}(M=1)=\kappa_{\mathrm{cl}}$ and the non-localized contribution is exactly $\delta_{\mathrm{spill}}^{\mathrm{cl}}$.
\end{proof}

Exact preservation follows when both terms are zero, but this is not assumed.
The probability-residual branch is calibrated under a finite empirical activation budget on the safety-approved reference population. This budget
does not constitute a population-level guarantee on future clean data.

%%%
\subsubsection{Dependency summary}
\label{app:dependencies}

\begin{table}[h]
\centering
\small
\caption{Role and scope of the theoretical results.}
\label{tab:app_dependencies}
\begin{tabular}{lll}
\toprule
Result & Main condition & Consequence \\
\midrule
Lemma~\ref{lem:app_simplex} & local directed transport & pair margin decreases \\
Cor.~\ref{cor:app_pair_flip} & positive initial pair margin & pairwise preference flips \\
Lemma~\ref{lem:app_inverse_operator} & fixed feasible transport amount & local inverse exists \\
Lemma~\ref{lem:transport} & feasible one-edge shift & coupling representation \\
Lemma~\ref{lem:app_reference_mass} & evidence definition & surviving victim evidence \\
Prop.~\ref{prop:reference_to_localization} & pair-correct state and additive semantic arm & $\eta_{\mathrm{loc}}\leq\eta_{\mathrm{ref}}$ \\
Prop.~\ref{prop:app_pair_separation} & scalar evidence gap & sufficient scalar separation \\
Prop.~\ref{prop:app_sf_mapping_contract} & family viability and direction agreement & fail-closed mapping release \\
Prop.~\ref{prop:app_one_way} & monotone filtering & vetoed references cannot return \\
Lemma~\ref{lem:app_orthogonal} & nondegenerate pair axis & axis-invariant matching \\
Prop.~\ref{prop:minimal_pair_projection} & full pair-only correction & minimum-$\ell_2$ substep \\
Prop.~\ref{prop:residual_asr} & localization/repair terms & residual-ASR bound \\
Cor.~\ref{cor:app_sf_no_new_target} & target-correct state and transformer support $\subseteq A$ & $\delta_{\mathrm{new}}=0$ \\
Prop.~\ref{prop:clean_change} & selective intervention & clean-change bound \\
\bottomrule
\end{tabular}
\end{table}

\begin{remark}[Scope of the theory]
The analysis does not claim that a real backdoor changes only two class probabilities, that its internal effect is one-dimensional, or that the unknown attack operator is observed exactly.
The inverse operator and the transposed coupling are distinct mathematical objects, as stated in the
main text and formalized above. 
It also does not claim that the safety-approved reference population is genuinely clean, or that a scalar
population pair score reproduces the complete transformers mapping selector.
The transport model is a local abstraction and the coupling is feasible rather than optimal. 
Pair-separation results are sufficient, not necessary. 
Proposition~\ref{prop:app_sf_mapping_contract} states a deterministic release contract conditional on observable builder evidence and is not a distribution-free theorem that the released pair equals the true attacked pair.
Calibration budgets are empirical reference-population constraints.
Clean preservation depends on localization and spillover rather than on an oracle clean-input gate.
\end{remark}

\subsection{Additional Method Details}
\label{app:method_details}

This subsection specifies the implementation details required to reproduce the operational pipeline in Section~\ref{sec:method}. 
Throughout this section, $(v,t)$ denotes the operational victim-target pair released by the attack-unaware builder. 
In the preceding theoretical analysis, $(v,t)$ instead denotes the true attacked pair, while $(\hat v,\hat t)$ denotes the released pair.

We use \emph{safety-approved} or \emph{clean-like} reference rather than \emph{clean reference}, since the defense has no access to true clean/poison provenance. 

%%%
\subsubsection{Information boundary and state construction}
\label{app:information_boundary}

RePAIR receives a frozen segmentation model and an unlabeled image pool that
may contain triggered samples.
Semantic ground truth, clean/poison provenance, attack masks, clean twins, attack identity, poisoning rate, and the true victim--target pair are excluded from all defense decisions.
Annotation fields are removed before every forward pass contributing to pair discovery, reference construction, calibration, localization, or repair.
Oracle information is accessed only after the official state has been frozen and serialized and is used exclusively for evaluation and diagnostics.

State construction proceeds through three fail-closed stages: attack-unaware mapping inference, construction of an image-disjoint safety-approved reference population, and reconstruction of the frozen runtime state from that population.
Failure to establish sufficient mapping confidence, reference capacity, or semantic coverage leaves the defense inactive.

%%%
\subsubsection{Architecture-specific semantic evidence}
\label{app:semantic_evidence}

For CNNs, normalized features from the final backbone representation define the CNN semantic reference.
For head-pseudo-labeled reference pixels $\mathcal P_k$, class centroids are
\begin{equation}
    \mu_k
    =
    \frac{1}{|\mathcal P_k|}
    \sum_{x\in\mathcal P_k}z(x) \enspace,
    \qquad
    \Yproto(x)
    =
    \arg\max_k
    \operatorname{sim}(z(x),\mu_k) \enspace.
    \label{eq:app_cnn_reference}
\end{equation}
Directed internal-to-head disagreement for candidate pair $(v,t)$ is measured by
\begin{equation}
    G_i^{\mathrm{CNN}}(v,t)
    =
    \frac{1}{|\Omega_i|}
    \sum_{x\in\Omega_i}
    \indic{
        \Yproto(x)=v
        \wedge
        \Yhead(x)=t
    } \enspace.
    \label{eq:app_cnn_pair_flow}
\end{equation}

The final builder resolves the target before selecting the victim. Once the target is confidence-resolved, victim proposals are restricted to that target and combine the established pair-evidence score with complementary victim-retention and victim-interior cues. 
A proposal is accepted only when it is sufficiently separated from competing candidates, otherwise mapping fails closed. 
No class identifier or attack-specific pair is hard-coded. 
Every selected pair must subsequently pass the unchanged reference-safety pipeline.

Transformers combine semantic evidence from $L$ selected internal representations. 
For representation $\ell$, victim and target centroid similarities are
\begin{equation}
    c_{\ell,v}(x)
    =
    \operatorname{sim}(z_\ell(x),\mu_{\ell,v}) \enspace,
    \qquad
    c_{\ell,t}(x)
    =
    \operatorname{sim}(z_\ell(x),\mu_{\ell,t}) \enspace.
    \label{eq:app_sf_layer_similarity}
\end{equation}
The selected layers are concatenated into
\begin{equation}
    \phi_{v,t}(x)
    =
    \big[
        c_{1,v}(x),\ldots,c_{L,v}(x),
        c_{1,t}(x),\ldots,c_{L,t}(x)
    \big]^\top \enspace,
    \label{eq:app_sf_joint_features}
\end{equation}
and a regularized logistic readout
\begin{equation}
    r_{v,t}(x)
    =
    \sigma
    \left(
        w^\top\widetilde{\phi}_{v,t}(x)+b
    \right)
    \label{eq:app_sf_logreg}
\end{equation}
provides the internal target compatibility. 
where $w$ and $b$ denote the learned weight vector and bias of the regularized logistic readout, respectively, and $\widetilde{\phi}_{v,t}(x)$ is the standardized joint multi-layer feature representation for pair $(v,t)$.
The corresponding victim support is
\begin{equation}
    p_v^{\mathrm{int}}(x)=1-r_{v,t}(x) \enspace.
    \label{eq:app_sf_victim_probability}
\end{equation}
Training labels for this readout are generated exclusively from frozen head pseudo-labels.

Transformers mapping first forms stable undirected family proposals from structural, mechanism, and family-verification evidence. 
A frozen rank prior is retained whenever its candidate passes the base viability test. 
If it is not viable, replacement is permitted only through the rescue path, which requires
bilateral leave-one-out support together with bilateral structural consistency.
Direction is then authorized within the frozen family using FullEdge, FamilyVerify, and an eligible Mechanism vote. 
Structural evidence does not act as a regular direction vote and is used only to corroborate the rescue rule.
Insufficient family viability or direction evidence leaves the mapping inactive. 
The formal release contract is given in Proposition~\ref{prop:app_sf_mapping_contract}.

%%%
\subsubsection{Reference-state safety}
\label{app:reference_state_safety}

Let $\mathcal C_0$ denote the provisional reference population and $V_i\in\{0,1\}$ the victim-safety veto. The safety-approved population is
\begin{equation}
    \mathcal C_{\mathrm{safe}}
    =
    \{i\in\mathcal C_0:V_i=0\} \enspace.
    \label{eq:app_reference_safe_set}
\end{equation}
All later repair banks satisfy
\begin{equation}
    \mathcal C_{\mathrm{bank}}
    \subseteq
    \mathcal C_{\mathrm{safe}}
    \subseteq
    \mathcal C_0 \enspace.
    \label{eq:app_reference_one_way}
\end{equation}
Thus, a vetoed reference cannot be reintroduced by refinement, capacity restriction, fallback, or backfilling.

Safety statistics are estimated image-disjointly. 
References with severe victim-evidence collapse, or moderate collapse accompanied by unusually
strong victim-to-target rupture, are rejected. 
State construction requires at least $24$ safety-approved images and at least $75\%$ semantic-class
coverage.

Capacity restriction is also fail-closed.
If a compressed reference core loses sufficient victim or target prototype support, the implementation may
retry the same capacity check on the larger already safety-approved
reservoir. 
This fallback does not relax thresholds, re-admit vetoed samples, or introduce backfill.

%%%
\subsubsection{Localization and repair}
\label{app:localization_repair}

For CNNs, runtime localization combines direct semantic contradiction, reference-calibrated geometric residuals, and a semantic probability residual:
\begin{equation}
    M_{\mathrm{CNN}}
    =
    M_{\mathrm{semantic}}
    \vee
    M_{\mathrm{geom}}
    \vee
    M_{\mathrm{prob}} \enspace.
    \label{eq:app_final_cnn_mask}
\end{equation}

For transformers, the final support is the sequential union
\begin{equation}
    M_{\mathrm{TF}}
    =
    M_{\mathrm{semantic}}
    \vee
    M_{\mathrm{src}}
    \vee
    M_{\mathrm{tgt}}
    \vee
    M_{\mathrm{traj}}
    \vee
    M_{\mathrm{prob}} \enspace,
    \label{eq:app_final_sf_mask}
\end{equation}
with
\begin{equation}
    M_{\mathrm{TF}}
    \subseteq
    \{\Yhead=t\} \enspace.
    \label{eq:app_sf_head_target_support}
\end{equation}

Both architectures use the same final probability-residual principle.
For discovered pair $(v,t)$, define the target-over-victim probability excess
\begin{equation}
    d_p(x)=q_t(x)-q_v(x) \enspace.
\end{equation}
Relative to the expected excess $\widehat d_{\mathrm{ref}}(x)$ obtained from the safety-approved semantic reference,
\begin{equation}
    D_{\mathrm{prob}}(x)
    =
    \left[
        d_p(x)-\widehat d_{\mathrm{ref}}(x)
    \right]_+ \enspace.
    \label{eq:app_probability_residual}
\end{equation}
The residual, $\delta_{\mathrm{prob}}$, is calibrated in four image-disjoint folds under the fixed empirical reference-pixel budget.

CNNs repair localized features from the final backbone layer toward matched safety-approved victim references and rerun the unchanged remainder of the network.
Transformers transports localized pixels toward its frozen multi-layer victim manifold.
Both architectures may subsequently apply a conservative pair-specific output correction. 
No network parameter is retrained or permanently modified.

Finally, feature repair may influence predictions outside directly localized cells.
Let $\Yrep$ denote the prediction immediately before the output veto and $\mathcal S=\{x:M(x)=1\}$.
RePAIR reverts to the undefended prediction whenever
\begin{equation}
    x\notin\mathcal S \enspace,
    \qquad
    \Yhead(x)\neq t \enspace,
    \qquad
    \Yrep(x)=t \enspace.
    \label{eq:app_veto}
\end{equation}
Thus, the repair cannot introduce a new unsupported target prediction outside the localized support.

%%%%%%%%%%%%%%%%%%%%%%%%%%%%%%%%%%%%%%%%%%%%%%%%%

\subsection{Details on experimental setting}
\label{app:exp_setting}

%%%
\paragraph{Datasets.}
We consider three datasets for evaluation, namely Cityscapes \cite{cordts2016Cityscapes}, ADE20K \citep{zhou2017scene}, and PASCAL VOC 2012 \citep{everingham2015pascal}. 
Cityscapes consists of $2,\!975$ training and $500$ validation images with a fixed resolution of $2048\times1024$ pixels and defines $19$ semantic classes for evaluation. 
The dataset focuses on densely annotated urban street scenes and contains a diverse set of simultaneously occurring object and stuff classes, such as roads, buildings, vegetation, pedestrians, and vehicles.

ADE20K is a substantially larger and more scene-centric benchmark, comprising $20,\!210$ training and $2,\!000$ validation images with variable native resolutions and dense pixel-level annotations across $150$ semantic classes, including both object and stuff categories.

We further evaluate RePAIR on PASCAL VOC 2012, whose semantic segmentation benchmark comprises $1,\!464$ training and $1,\!449$ validation images with variable image resolutions, typically around $500\times375$ pixels, and defines $20$ foreground classes plus a background class. 
In contrast to the more scene-centric Cityscapes and ADE20K datasets, PASCAL VOC is predominantly object-centric: images generally focus on one or more salient foreground objects belonging to the predefined object classes, while pixels not assigned to these classes are represented by a single background label.

Since both, PASCAL VOC and ADE20K, contain images with varying native resolutions, we resize all images to $512\times512$ pixels. 
We use ADE20K as an attack-effectiveness negative control. since the evaluated attacks do not produce a sufficiently strong backdoor effect on this dataset, RePAIR is not applied.

%%%
\paragraph{Segmentation models.} 
We evaluate two representative segmentation architectures: DeepLabV3+ with a ResNet-50 backbone \citep{Chen2018_deeplab} and SegFormer with a MiT-B5 backbone \citep{xie2021segformersimpleefficientdesign}. 
DeepLabV3+ is trained for 80,000 iterations on Cityscapes and ADE20K and for 40,000 iterations on PASCAL VOC 2012, while SegFormer is trained for 160,000 iterations on all datasets.
The other training parameters are as specified in the repository\footnote{\url{https://github.com/open-mmlab/mmsegmentation}}.

%%%
\paragraph{Backdoor attacks.}
We consider IBA with NNI and PRL \citep{lan2024influencer}, ConSeg \citep{abbasi2025consegcontextualbackdoorattack}, OFBA \citep{mao2023object}, and BadSeg \citep{zhang2026poisoningpixelsrevisitingbackdoor}. 
All attacks are trained with a $10\%$ poisoning rate, and trigger configurations follow the corresponding attack papers. 

For IBA, we use the Hello Kitty trigger and the victim-to-target relabeling described in the original method. 
IBA-NNI places the trigger on non-victim pixels close to the victim region, using a $15\times15$ trigger on PASCAL VOC and a $55\times55$ trigger on Cityscapes. 
IBA-PRL additionally applies Pixel Random Labeling (PRL), where $50,\!000$ additional pixels are assigned randomly sampled labels from the same segmentation mask. 

OFBA uses a $29\times29$ black-and-white checkerboard trigger with visibility coefficient $\lambda=0.4$, placed entirely within an object region.
The corresponding object is relabeled to the target class, while the attacked object and trigger position are varied during training. 

ConSeg uses a naturally occurring object as a semantic trigger and relabels the victim class as the target class. 
It additionally modifies victim-region labels using target-related co-occurring classes, with five such classes and four replaced pixels per class in the default configuration.

BadSeg jointly optimizes the trigger and victim-target pair using a surrogate segmentation model. We evaluate its object-to-object (O20), object-to-background (O2B), background-to-object (B2O), and background-to-background (B2B) configurations.

%%%
\paragraph{Baseline defenses.}
For the defense benchmarks, we use the method-specific parameters listed in Table~\ref{tab:defense_hyperparameters}, following their original configurations.

\begin{table}[t]
\centering
\caption{Hyperparameters of the benchmark defenses.}
\label{tab:defense_hyperparameters}
\small
\begin{tabular}{lll}
\toprule
\textbf{Defense} & \textbf{Parameter} & \textbf{Value} \\
\midrule

Fine-tuning
& Clean data fraction & $10\%$ \\
& Epochs & 10 \\
\midrule

Pruning
& Pruning ratio & $10\%$ \\
& Criterion & Least-active final-layer channels \\
\midrule

ABL
& Isolation epochs & 20 \\
& Isolation ratio $p$ & 0.01 \\
& $\gamma$ & 0.5 \\
& Batch / scoring batch & $8/8$ \\
& Unlearning epochs & 5 \\
& Unlearning optimizer & SGD \\
& Unlearning learning rate & $5\times10^{-4}$ \\
& Recovery epochs & 0 \\
& Clean-loss guard & 32 samples, threshold $3\times$ \\
\midrule

Segment Cleanse
& Clean images & 1000 \\
& Working resolution & $512\times1024$ \\
& Target knowledge & Unknown; all alternatives scanned \\
& Reverse optimizer & Adam \\
& Reverse learning rate & $5\times10^{-3}$ \\
& Regularization $\alpha$ & $10^{-4}$ \\
& Maximum reverse epochs & 100 \\
& Early-stop warm-up / patience & $10/5$ \\
& Relative threshold & $10^{-4}$ \\
& Parallel target candidates & 2 \\
& Mitigation & Decoder-only \\
& Mitigation optimizer & SGD \\
& Mitigation learning rate & $10^{-4}$ \\
& Mitigation epochs & 20 \\
\bottomrule
\end{tabular}
\end{table}

%%%
\paragraph{Evaluation metrics.}

We evaluate attack effectiveness and segmentation utility using Attack Success Rate (ASR), benign mean intersection-over-union (mIoU-B), and attacked mIoU (mIoU-A). 
ASR measures the fraction of ground-truth victim-class pixels on triggered inputs that are predicted as the attacker-selected target class. For BadSeg-B2O, where the attack is restricted to a designated spatial region, ASR is computed only within the corresponding evaluation region.

mIoU-A measures the mean IoU on triggered inputs after excluding the pixels used to evaluate the attack. For class-wide attacks, all ground-truth victim-class pixels are excluded, whereas for BadSeg-B2O only the designated attack region is excluded. mIoU-B measures the mean IoU on clean inputs. 
In the literature, mIoU-A and mIoU-B are also referred to as Poisoned Benign Accuracy (PBA) and Clean Benign Accuracy (CBA), respectively \citep{lan2024influencer,abbasi2025consegcontextualbackdoorattack}.

Lower ASR indicates stronger attack mitigation, while higher mIoU-A and mIoU-B indicate better segmentation utility.

%%%
\paragraph{Hyperparameter selection. }
\label{app:hyperparameters}

All RePAIR hyperparameters are treated as global method settings and fixed independently of the evaluated attack instance. We do not select parameters based on attack identity, the oracle victim-target pair, clean/poison provenance, attack masks, ground-truth segmentation labels, or evaluation ASR. 
Structural parameters, such as minimum reference capacity, semantic coverage requirements, retrieval sizes, and calibration budgets, are fixed across attacks. 
Data-dependent thresholds are reconstructed only from the frozen safety-approved reference population using image-disjoint calibration and remain fixed during evaluation.
Thus, attack-specific test performance is never used to tune the repair state.

Run-specific preprocessing, checkpoints, random seeds, and execution parameters are provided in the experiment configurations summarized in Table~\ref{tab:app_hyperparameters}.

\begin{table}[htbp]
\centering
\small
\caption{Principal global state-construction and inference parameters for Cityscapes.
}
\label{tab:app_hyperparameters}
\begin{tabular}{@{}llr@{}}
\toprule
Parameter & Variable & Value \\
\midrule
\multicolumn{3}{@{}l}{\textit{Shared}} \\
\quad Image-disjoint safety/calibration folds & $K_{\mathrm{fold}}$            & $4$ \\
\quad Minimum safety-approved images          & $N_{\mathrm{safe}}^{\min}$     & $24$ \\
\quad Minimum semantic class coverage         & $\gamma_{\mathrm{cov}}$        & $75\%$ \\
\quad Semantic probability budget             & $\delta_{\mathrm{prob}}$       & $0.05$ \\
\midrule
\multicolumn{3}{@{}l}{\textit{DeeplabV3+}} \\
\quad Target-authority winner ratio           & $\tau_{\mathrm{tgt}}$          & $4.00$ \\
\quad Legacy victim winner ratio              & $\tau_{\mathrm{vic}}$          & $1.20$ \\
\quad Legacy victim robust MAD-$z$            & $z_{\mathrm{vic}}^{\min}$      & $1.00$ \\
\quad Geometry-fallback robust MAD-$z$        & $z_{\mathrm{geo}}^{\min}$      & $2.00$ \\
\quad Reference-bank cap (per relevant class) & $N_{\mathrm{ref}}^{\max}$      & $4096$ \\
\quad Self-counterfactual strength            & $\lambda_{\mathrm{cf}}$        & $0.25$ \\
\quad Base residual allowance                 & $\delta_{\mathrm{base}}$       & $0.02$ \\
\quad Auxiliary residual allowance            & $\delta_{\mathrm{aux}}$        & $0.005$ \\
\midrule
\multicolumn{3}{@{}l}{\textit{SegFormer}} \\
\quad Semantic victim-support quantile        & $q_{\mathrm{sem}}$             & $0.995$ \\
\quad Repair-margin quantile
& $q_{\mathrm{rep}}$ & $0.10$ \\
\quad Repair-margin quantile                  & $q_{\mathrm{rep}}$             & $0.10$ \\
\quad Minimum suspicious OOF group (images)   & $N_{\mathrm{OOF}}^{\min}$      & $12$ \\
\quad Manifold retrieval neighbors            & $k_{\mathrm{man}}$             & $8$ \\
\quad Trajectory residual budget              & $\delta_{\mathrm{traj}}$       & $0.005$ \\
\bottomrule
\end{tabular}
\end{table}

%%%%%%%%%%%%%%%%%%%%%%%%%%%%%%%%%%%%%%%%%%%%%%%%%% 
\subsection{Additional experimental results}
\label{app:datasets}

%%%
\paragraph{ADE20K. }
\label{app:ade20k}

ADE20K contains 150 semantic classes and constitutes a substantially
different semantic setting from Cityscapes and PASCAL VOC. In our
experiments, the evaluated attack configurations do not establish a
meaningful backdoor effect in the undefended ADE20K models. The victim-target pair bein used on ADE20K is car-road.

We therefore do not treat ADE20K as an additional defense benchmark.
A defense result is only informative when the underlying attack first achieves a substantive victim-to-target effect, reducing an already negligible ASR would not demonstrate successful backdoor removal.

Accordingly, the official ADE20K evaluation is baseline-only. 
RePAIR state construction and repair are not applied. We report the undefended ASR to document the attack effectiveness of the evaluated configurations.

\begin{table}[t]
\caption{\textbf{ADE20K Attack Results} Undefended attack success rate (ASR, \%). RePAIR is not applied because the evaluated attacks do not establish a meaningful backdoor effect.}
\label{tab:ade20k}
\centering
\small
\renewcommand{\arraystretch}{1.0} 
\begin{tabular}{llccccc}
\toprule
Model & Method & IBA (NNI) & IBA (PRL) & ConSeg & OFBA & BadSeg-O2B \\
\midrule
DeepLabV3+ & no defense
& $3.02$
& $3.22$
& $2.88$
& $2.40$
& $2.58$ \\
\midrule
SegFormer & no defense
& $2.51$
& $2.38$
& $4.93$
& $2.48$
& $2.43$ \\
\bottomrule
\end{tabular}
%}
\end{table}

The ADE20K results should therefore be interpreted as an attack-effectiveness control rather than as evidence of defense performance.
They show that, under the evaluated configurations, the prerequisite backdoor behavior does not
emerge strongly enough to support a meaningful mitigation experiment.

%%%
\paragraph{PASCAL VOC 2012.}
PASCAL VOC provides a challenging transfer setting because the background class occupies a large fraction of many images and represents a heterogeneous collection of visual content rather than a compact semantic object class.
Consequently, natural victim-to-background relations can compete with the directed evidence produced by an O2B backdoor.

Unlike the Cityscapes experiments, all attacks considered on VOC use the same O2B-style semantic relation.
The victim--target pair is used only to compute the reported ASR and is never supplied to RePAIR. The chosen victim-target pair for VOC is person-background.

\begin{table}[t]
\caption{\textbf{PASCAL VOC 2012.} Undefended and defended attack success
rate (ASR, \%).}
\label{tab:voc_asr}
\centering
\small
\setlength{\tabcolsep}{5pt}
\renewcommand{\arraystretch}{1.2}
\begin{tabular}{llrr}
\toprule
Model & Attack & No defense & RePAIR \\
\midrule
\multirow{5}{*}{\shortstack[l]{DeeplabV3+}}
 & IBA (NNI)   & $12.21$ & $5.41$ \\
 & IBA (PRL)   & $15.03$ & $5.45$ \\
 & ConSeg      & $16.12$ & $7.05$ \\
 & OFBA        & $99.71$ & $8.51$ \\
 & BADSEG-O2B  & $92.14$ & $19.67$ \\
\midrule
\multirow{5}{*}{\shortstack[l]{SegFormer}}
 & IBA (NNI)   & $4.83$ & $0.12$ \\
 & IBA (PRL)   & $7.26$ & $0.19$ \\
 & ConSeg      & $6.67$ & $0.15$ \\
 & OFBA        & $98.42$ & $0.50$ \\
 & BADSEG-O2B  & $95.78$ & $0.34$ \\
\bottomrule
\end{tabular}
\end{table}

Table~\ref{tab:voc_asr} evaluates whether the attack-unaware state builder can recover a usable repair state despite the strongly background-dominated semantic distribution of VOC. 
The reported defended values correspond only to states released by the fail-closed construction procedure. No oracle victim or target information is used during state construction.

To quantify the utility retained after repair, we additionally report clean and attacked segmentation performance in Table~\ref{tab:voc_miou}.

\begin{table}[t]
\caption{\textbf{PASCAL VOC 2012 segmentation utility.} Mean IoU on clean
inputs (mIoU-B) and triggered inputs (mIoU-A).}
\label{tab:voc_miou}
\centering
\small
\begin{tabular}{lrr}
\toprule
Model & mIoU-A & mIoU-B \\
\midrule
DeeplabV3+ / ResNet-50
& $66.24_{-4.99}^{+2.71}$
& $58.21_{-2.05}^{+1.48}$ \\
SegFormer / MiT-B5
& $84.45_{-3.34}^{+1.28}$
& $76.17_{-0.62}^{+0.22}$ \\
\bottomrule
\end{tabular}
\end{table}

\end{document}